\documentclass[twoside,11pt]{article}
\usepackage{tikz}
\usetikzlibrary{positioning,fit,arrows.meta,calc}
\usepackage{blindtext}
\usepackage{amsmath}
\usepackage{style}

\usepackage{lastpage}

\ShortHeadings{A VAE Approach to Sample Multivariate Extremes}{Lafon, Naveau and Fablet}
\firstpageno{1}

\begin{document}

\title{A VAE Approach to Sample Multivariate Extremes}

\author{\name Nicolas Lafon \email nicolas.lafon@lsce.ipsl.fr \\
       \addr Laboratoire des Sciences du Climat et de l'Environnement\\
       ESTIMR, CNRS-CEA-UVSQ\\
       Gif-sur-Yvette, France
       \AND
       \name Philippe Naveau \email philippe.naveau@lsce.ipsl.fr \\
       \addr Laboratoire des Sciences du Climat et de l'Environnement\\
       ESTIMR, CNRS-CEA-UVSQ\\
       Gif-sur-Yvette, France
       \AND 
       \name Ronan Fablet \email ronan.fablet@imt-atlantique.fr \\
       \addr IMT Atlantique, Lab-STICC\\
       ODISSEY, INRIA\\
       Brest, France}

\maketitle

\begin{abstract}
 Generating accurate extremes from an observational data set is crucial when seeking to estimate risks associated with the occurrence of future extremes which could be larger than those already observed. Applications range from the occurrence of natural disasters to financial crashes. Generative models from the machine learning (ML) community do not apply to extreme samples without careful adaptation. Besides, asymptotic results from extreme value theory (EVT) give a theoretical framework to model multivariate extreme events. Bridging these two fields, this paper details a variational autoencoder (VAE) approach for sampling multivariate heavy-tailed distributions, in which extremes of particularly large intensity are likely to occur. We illustrate the relevance of our approach on a synthetic data set and on a real data set of discharge measurements along the Danube river network. The latter shows the potential of our approach for flood risks' assessment. In addition to outperforming the vanilla VAE for the tested data sets, we also provide a comparison with a competing EVT-based generative approach. In the tested cases, our approach better captures the dependence structure between extreme events.

\end{abstract}


\section{Introduction}
Simulating samples from an unknown distribution is a task that various studies have successfully tackled in the ML community during the past decade. This has led to the emergence of generative models, such as generative adversarial networks (GANs) \citep{goodfellow2020generative}, VAEs \citep{kingma2013auto, rezende2014stochastic}, or normalizing flows \citep{rezende2015variational} (NFs). 
As ML tasks usually focus on average behaviors rather than rare events, these methods were originally not tailored to extrapolate upon the largest value of the training data set. This is a major shortcoming when dealing with extremes. Risk assessment in worst-case scenarios requires accurately sampling extremes at large quantiles, beyond the maximum values observed in the data set \citep{embrechts1999extreme}. Figure 1 illustrates this challenge with a two-dimensional case study: using a VAE approach, we aim, when relevant, to consistently generate samples in the extreme region (black square) from observations (blue dots), none of which lie within that region. In this context, EVT characterizes the probabilistic structure of extreme events and provides a theoretically-sound statistical framework to analyze them. A central concept in this framework is heavy-tail analysis \citep{resnick2007heavy}, which investigates phenomena where extreme values occur with non-negligible probability. Heavy-tailed distributions are thus particularly relevant in applications where rare but severe events play a crucial role. Data modeled by heavy-tailed distributions cover a wide range of application fields, e.g., hydrology \citep{anderson1998modeling,rietsch2013network}, particle motion \citep{fortin2015applications}, finance \citep{bradley2003financial}, Internet traffic \citep{hernandez2004variable}, and risk management \citep{chavez2004extreme,das2013four}.\\
\begin{figure}\label{fig: pb extreme}
\centering
    \includegraphics[width=0.6\columnwidth]{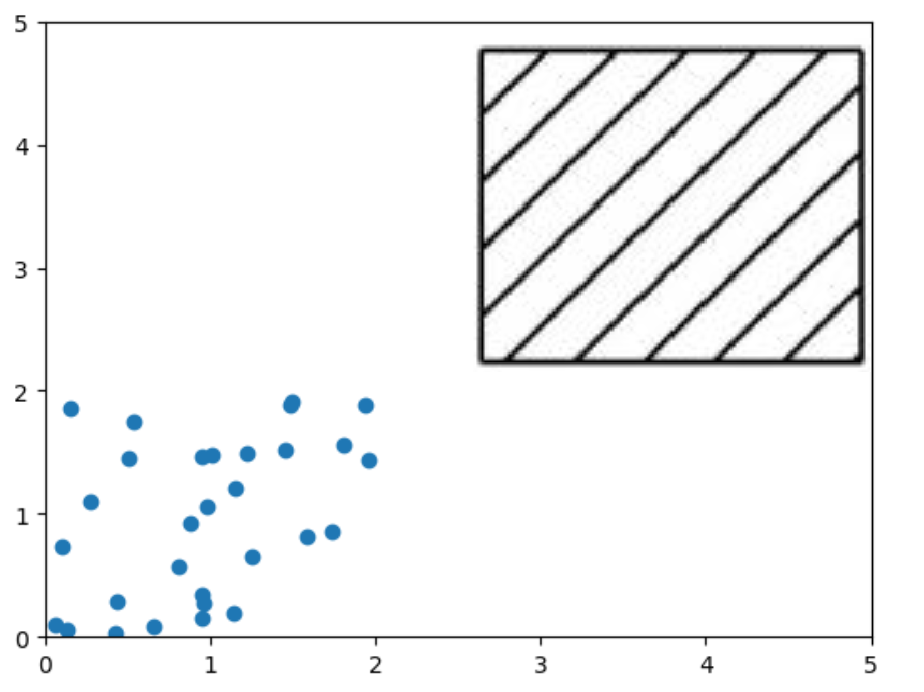}
    \label{fig: Ext regions}
\caption{How to sample from observations (blue dots) in extreme regions (black square) to estimate probability of rare events?}
\label{fig: QQ plots}
\end{figure}

Recently, the intersection of ML and EVT has attracted growing attention. Several studies have highlighted the benefits of combining the two fields for diverse tasks. Beyond the generation of extremes, which is further discussed in the next paragraph, applications include dimensionality reduction \citep{drees2021principal}, quantile function approximation \citep{pasche2022neural}, outlier detection \citep{rudd2017extreme}, and classification in tail regions \citep{jalalzai2018binary}.\\ 

\noindent
\textbf{Related works:} As demonstrated in \cite{jaini2020tails,huster2021pareto}, the output random variable of a neural network associated with a light-tailed input random variable cannot be heavy-tailed. This limitation has been the main motivation for adapting GANs and NFs to the sampling of heavy-tailed distributions.
The main line of research relies on introducing heavy-tailed latent variables. This strategy has been implemented through GANs with Pareto priors in \cite{feder2020nonlinear,huster2021pareto}, as well as NFs with Student-t priors in \cite{jaini2020tails,laszkiewicz2022marginal}, the latter allowing for marginals with heterogeneous tail behaviors.
A second approach treats marginal distributions and dependence structure separately. In \cite{boulaguiem2022modeling}, the marginals are first estimated by fitting generalized Pareto distributions. To learn the dependence structure, the marginals are transformed to uniform distributions, and a GAN is trained in this transformed space, effectively capturing the copula representation \citep{embrechts2009copulas}.
Another strategy targets the approximation of heavy-tailed quantile functions. Since neural networks with Rectified Linear Unit (ReLU) activations cannot directly map $[0,1]$ to the quantile function of a heavy-tailed distribution, \cite{allouche2022ev} proposed a GAN generator tailored to this mapping. Their results support the theoretical soundness of this approach for approximating true quantile functions.
Finally, an empirical direction has been proposed in \cite{bhatia2021exgan}, where GANs are recursively trained on tail samples to progressively reach a targeted return level.\\

\noindent
\textbf{Main contributions:} 
To our knowledge, our study is the first attempt to bridge VAE and EVT. Recent studies suggest that state-of-the-art likelihood-based models, including VAEs, may, in some examples, capture the spread of the true distribution better than GANs \citep[see, e.g.,][]{razavi2019generating,nash2021generating}. This makes VAE an interesting way of explicitly exploiting the EVT framework to generate realistic and diverse extremes.
Our main contributions are as follows:
\begin{itemize}
    \item We first demonstrate that a VAE with standard parameterization cannot generate heavy-tailed distributions. We then propose a VAE framework to sample extremes from heavy-tailed distributions.

\item Our approach can capture complex dependence structures between extremes by leveraging a polar representation of the data. It also enables direct sampling from the angular measure, which characterizes the asymptotic dependence between extremes (see Section \ref{mve})

\item We illustrate the ability of our approach to generate realistic extremes in numerical experiments on both synthetic and real data sets, outperforming the Vanilla VAE and an EVT-based GAN approach \citep{huster2021pareto}. 

\end{itemize}


\noindent
\textbf{Organization of the paper:} In Section \ref{background}, we recall the basic principles of VAE and EVT. In Section 3, we present our main theoretical results concerning, on the one hand, the tail distribution of the marginals generated by a VAE and, on the other hand, the angular measure of generative models. All detailed proofs are delayed in Appendix \ref{app: proofs}, listed in order of appearance in the paper. 
We detail the proposed VAE framework for multivariate extremes in Section \ref{vif} and describe the associated training setting in Section \ref{lf}. Section \ref{experiments} is dedicated to experiments. Section \ref{ccl} is devoted to concluding remarks.

\section{Background}\label{background}
In this section, we present background knowledge about VAE, and we give an introduction to univariate and multivariate heavy-tailed distributions.

\subsection{VAE framework}\label{amortize VI}

To generate a sample from a random vector $\mathbf{X}$, a VAE relies on a two-step sampling strategy:
\begin{itemize}
\item A latent vector $\mathbf{z}$ is drawn from a prior distribution $p_\alpha(\mathbf{z})$, parameterized by $\alpha$;
\item A sample $\mathbf{x}$ is then drawn from the conditional distribution $p(\mathbf{x}\mid \mathbf{z})$.
\end{itemize}
Since $p(\mathbf{x}\mid\mathbf{z})$ is in general not known, a parametric approximation $p_\theta(\mathbf{x}\mid\mathbf{z})$, referred to as the likelihood or probabilistic decoder, is introduced, with $\theta$ a set of learnable parameters. The aim is to find a parameterization that enables the generation of realistic samples of $\mathbf{X}$. To this end, the VAE framework introduces a variational distribution (or probabilistic encoder) $q_\phi(\mathbf{z}\mid \mathbf{x})$, parameterized by $\phi$, to approximate the true posterior distribution $p(\mathbf{z}\mid \mathbf{x})$. Training then amounts to maximizing the evidence lower bound (ELBO) with respect to the parameters $(\alpha, \phi, \theta)$. Formally, given $N$ independent samples $(\mathbf{x}^{(i)})_{i=1}^{N}$ of $\mathbf{X}$, one has for each $i$,
\[
\log p(\mathbf{x}^{(i)}) \geq L(\mathbf{x}^{(i)},\alpha,\theta,\phi),
\]
where the ELBO $L$ is defined as
\begin{equation}\label{eq: ELBO_VAE}
    L(\mathbf{x}^{(i)},\alpha,\theta,\phi) 
    = -D_{KL}\!\left(q_\phi(\mathbf{z}\mid \mathbf{x}^{(i)}) \,\|\, p_\alpha(\mathbf{z})\right)
    + \mathbb{E}_{q_\phi(\mathbf{z}\mid \mathbf{x}^{(i)})}\!\left[\log p_{\theta}(\mathbf{x}^{(i)} \mid \mathbf{z})\right].
\end{equation}
The operator $D_{KL}$ denotes the Kullback--Leibler (KL) divergence, defined for two distributions $q$ and $p$ by
\[
D_{KL}(q \,\|\, p) = \int q(z)\,\log\!\left(\tfrac{q(z)}{p(z)}\right)\,dz,
\]
which is always non-negative and equals zero if and only if $q = p$ almost everywhere. The ELBO over the whole data set is obtained by averaging Equation~\eqref{eq: ELBO_VAE} over the $N$ samples of $\mathbf{X}$.\\

To optimize the parameters $(\alpha, \phi, \theta)$, \cite{kingma2013auto} and \cite{rezende2014stochastic} introduced a training scheme that makes Equation~\eqref{eq: ELBO_VAE} differentiable with respect to the parameters. First, the KL divergence term often admits a closed-form expression and is therefore easily differentiable with respect to $(\alpha, \phi)$. Second, the expectation in Equation~\eqref{eq: ELBO_VAE} is approximated by an unbiased Monte Carlo estimator,
\[
\frac{1}{L}\sum_{l=1}^L \log p_{\theta}(\mathbf{x}^{(i)} \mid \mathbf{z}^{(i,l)}),
\]
where $\mathbf{z}^{(i,l)}$ are sampled from the random vector $\tilde{\mathbf{Z}}^{(i)}$ with distribution $q_\phi(\mathbf{z} \mid \mathbf{x}^{(i)})$. Differentiability is ensured by an explicit reparameterization, known as reparameterization trick, which introduces a function $g_{\phi}$, differentiable with respect to $\phi$, such that
\begin{equation}\label{eq: reparametrization}
   \tilde{\mathbf{Z}}^{(i)}  = g_{\phi} (\mathbf{x}^{(i)},\mathcal{E}),
\end{equation}
with $\mathcal{E}$ a suitable random variable. Beyond ensuring differentiability, this reparameterization yields a low-variance estimator of the gradient, which is crucial for stable and efficient training. When explicit reparameterization is not feasible, implicit reparameterization can be employed \citep[see][]{figurnov2018implicit}. Further details on implicit reparameterization are provided in Appendix~\ref{app: implicit}.

\begin{example}\label{ex: standard VAE}
The most common parameterization of a VAE with ${\bf z} \in \mathbb{R}^n$ and ${\bf x} \in \mathbb{R}^m$ is
\begin{align*}
    p({\bf z})&= \mathcal{N}({\bf z} ; {\bf 0},{\bf I}_{n}),\\
    p_\theta({\bf x}\mid{\bf z})&= \mathcal{N}\!\left({\bf x} ; {\bf \mu}_{\theta}({\bf z}),\operatorname{diag}({\bf \sigma}_\theta({\bf z}))^2\right),\\
    q_\phi({\bf z}\mid{\bf x})&= \mathcal{N}\!\left({\bf z} ; {\bf \mu}_{\phi}({\bf x}),\operatorname{diag}({\bf \sigma}_{\phi}({\bf x}))^2\right),
\end{align*}
where ${\bf \mu}_{\theta}$ and ${\bf \sigma}_{\theta}$ (resp. ${\bf \mu}_{\phi}$ and ${\bf \sigma}_{\phi}$) are neural network functions with parameters $\theta$ (resp. $\phi$); $\operatorname{diag}(\mathbf{v})$ denotes the diagonal matrix with the vector $\mathbf{v}$ on its diagonal; and $\mathcal{N}({\bf z};{\bf \mu},{\bf \Sigma})$ is the pdf in variable ${\bf z}$ of the multivariate normal distribution with mean ${\bf \mu}$ and covariance matrix ${\bf \Sigma}$. In this framework, the reparameterization trick takes the form
\[
g_{\phi} (\mathbf{x},\mathcal{E}) = {\bf \mu}_{\phi}({\bf x}) + {\bf \sigma}_{\phi}({\bf x})\odot\mathcal{E},
\]
where $\mathcal{E} \sim \mathcal{N}({\bf 0} , {\bf I}_n)$ and $\odot$ denotes the element-wise product. In this case, the KL divergence of Equation~\eqref{eq: ELBO_VAE} simplifies to
\[
D_{KL}\!\left(q_\phi(\mathbf{z}\mid \mathbf{x}^{(i)}) \,\|\, p(\mathbf{z})\right) 
= \tfrac{1}{2}\sum_{j=1}^n \left(1+2\log[{\bf \sigma}_{\phi}({\bf x}^{(i)})]_j -  [{\bf \mu}_{\phi}({\bf x}^{(i)})]_j^2 - [{\bf \sigma}_{\phi}({\bf x}^{(i)})]_j^2 \right),
\]
where $[\mathbf{v}]_j$ is the $j^{th}$ component of the vector $\mathbf{v}$.\\
We refer to this parameterization as the Standard VAE in the following.
\end{example}

\subsection{Univariate Extremes}\label{univariate}
When modelling univariate extremes, generalized Pareto (GP) \citep{pickands1975statistical} distributions are of great interest. The GP survival function is defined for $\xi \in \mathbb{R}$ and $\sigma >0$ by
\begin{equation}\label{eq: GP}
    \bar{H}_{\sigma,\xi} (x)= \left(1+\xi\frac{x}{\sigma}\right)_+^{-1/\xi},
\end{equation}
where $a_+=0$ if $a<0$. The scalar $\xi$ is called the shape parameter. Note that Equation \eqref{eq: GP} is extended to $\xi = 0$, with $\bar{H}_{\sigma,0}$ survival function of the exponential distribution of scale parameter $\sigma$.\\
Given a random variable $X$ with cdf $F$, GP distributions appear under mild conditions as the simple limit of the threshold exceedance function $F_u(x) = P(X-u\leq x \mid X>u)$ when $u\to \infty$ \citep{balkema1974residual}. To be explicit, there exists $\xi\in \mathbb{R}$ and a strictly positive function $\sigma(\cdot)$ such that
$$
 \lim_{u\to x_F}\sup_{0<x<x_F-u} \mid F_u(x)-H_{\sigma(u),\xi}(x)\mid = 0,
$$
with $x_F = \sup \{x\ \text{ s.t. } F(x)<1\}$ the right endpoint of $F$, and $H$ the cdf of the GP distribution.\\
The shape parameter $\xi$ of the GP approximation of $F_u$
encompasses the information about the tail of $X$. In the following, we consider that:
\begin{itemize}
    \item $\xi\leq 0$ corresponds to light-tailed distribution,
    \item $\xi>0$ corresponds to heavy-tailed distribution.
\end{itemize}
\begin{remark}\label{rk: sampling_GP}
    A simple yet efficient way to sample from a GP distribution with parameters $\xi$ and $\sigma$ is to multiply an inverse gamma distributed random variable 
with shape $\frac{1}{\xi}$ and rate $\sigma$ by a unit and independent exponential one. This multiplicative feature is 
essential for understanding
the pivotal role of inverse-Gamma random variables in our sampling scheme in Section \ref{multrad}. 
\end{remark}
\begin{remark}\label{rk: lt_ht}
    Notice that, given a light-tailed distribution with survival function $\bar{F}$, all its higher-order moments exist and are finite, and $\lim_{u\to \infty}u^a\bar{F}(u) = 0$ for any $a>0$. In particular, Gaussian distributions are light-tailed. At the contrary, not all higher-order moments are finite for a heavy-tailed distribution.  
\end{remark}
In this work, we focus on heavy-tailed distributions, which can be seen as the distributions for which extremes have the greater intensity. The shape parameter governs the heaviness of a distribution’s tail: the larger it is, the heavier the tail.
\\ 
A final important notion regarding extreme values is the so-called regular variation property. 
\begin{definition}\label{def: RV}
    A random variable $X$ is said to be regularly varying with tail index $\alpha > 0$, if
\begin{equation}\label{eq: regular variation}
    \lim_{t \to +\infty}P(X>tx\ \mid \ X>t) = x^{-\alpha}.
\end{equation}
\end{definition}
Importantly, $X$ regularly varying equates to $X$ heavy-tailed with $\alpha = \frac{1}{\xi}$ \citep[see][Theorem 8.13.2]{bingham1989regular}.

\subsection{Multivariate Extremes}\label{mve}
By extending the notions developed in Section \ref{univariate}, a multivariate analogue of the GP distribution (Equation \ref{eq: GP}), referred to as multivariate GP, can be defined \citep[see][]{rootzen2006multivariate}. Under mild conditions, the multivariate exceedance distribution asymptotically follows a multivariate GP distribution. Additionally, the regular variation of Definition \ref{def: RV} can be extended to a multivariate regular property \citep[see, e.g.][for details]{resnick2007heavy}. As in the univariate case, the convergence of exceedance distributions to a multivariate GP distribution and the property of multivariate regular variation are closely related. Thus, a multivariate random vector is heavy-tailed if it has multivariate regular variation.\\ 
Let $\mathbf{X}$ be a random vector in $\left(\mathbb{R}^+\right)^m$. To define multivariate regular variation, we decompose $\mathbf{X}$ into a radial component $R = X_1+\cdots+X_m = \|\mathbf{X}\|$ and an angular component of the $(m-1)$-dimensional simplex $\mathbf{\Theta} = \frac{\mathbf{X}}{\|{\bf X}\|}$. 
\begin{definition}\label{def: MRV}
    $\bf X$ has multivariate regular variation if the two following properties are fulfilled:
\begin{itemize}
    \item The radius $R$ is regularly varying as defined in Equation \eqref{eq: regular variation};
    \item There exists a probability measure $\mathbf{S}$ defined on the $(m-1)$-dimensional simplex such that $(R,\mathbf{\Theta})$ satisfies
    \begin{equation}\label{eq: wan MRV}
         P\left(\mathbf{\Theta}\in \bullet \bigm| R>r\right) \overset{w}{\longrightarrow} \mathbf{S}(\bullet),
     \end{equation}
     where $\overset{w}{\longrightarrow}$ denotes weak convergence (see Appendix \ref{app: weak_conv}). $\bf S$ is called angular measure.
\end{itemize}
\end{definition}
Consequently, the radius is a univariate heavy-tailed random variable as described in Section \ref{univariate}. The tail index of a multivariate regularly varying random vector is defined as the tail index of its radius. Besides, Equation \eqref{eq: wan MRV} indicates that, if the radius is above a sufficiently high threshold, the respective distributions of the radius and the angle can be considered independent. 
Estimating the angular measure then becomes crucial to address tail events of the kind of $\{\mathbf{X}\in C\}$ for an ensemble $C$ such that $u=\inf\{\|\mathbf{x}\|,\mathbf{x} \in C\}$ is large. This is especially true to assess the probability of joint extremes.\\ 
More generally, the estimation of the angular measure $\bf S$, although difficult due to the scarcity of examples \citep{clemenccon2021concentration}, is of great interest for the analysis of extreme values. In particular, it allows to determine confidence intervals for the probabilities of rare events \citep{de1998sea}, bounds for probabilities of joint excesses over high thresholds \citep{engelke2017robust} or tail quantile regions \citep{einmahl2013estimating}.

\section{Tail properties of Distributions Sampled by Generative Models} \label{tail properties}
Some important theoretical results are presented in this section. We first stress in Section \ref{margins VAE} that the Standard VAE cannot generate heavy-tailed marginals. Then we focus on angular measures that can be obtained using generative models in Section \ref{limit GANs}. In particular, we prove that, when restricted to neural networks with ReLU activation functions, generative models based on the deterministic transformation of a prior input (e.g. GANs or NFs) have an angular measure concentrated on a restricted number of vectors. These theoretical considerations are crucial to define our VAE approach presented in Section \ref{vif}.

\subsection{Marginal Tail of a Standard VAE}\label{margins VAE}
In this section, we establish that a Standard VAE only produces light-tailed marginals. This result extends to VAEs results similar to those established for GANs with normal prior \citep[see][]{huster2021pareto}, or with NFs with light-tailed base distribution \citep{jaini2020tails}. We first mention an important property of neural networks, based on the notion of Lipschitz continuity (see Appendix \ref{app: lip_continuity}).

\begin{proposition}\label{prop: PWL NN}
    \citep{arora2016understanding,huster2021pareto}: A neural network $f: \mathbb{R}^n \rightarrow \mathbb{R}$ composed of operations such as ReLUs, leaky ReLUs, linear layers, maxpooling, maxout activation, concatenation or addition, is a piecewise linear operator with a finite number of linear regions. Therefore, $f$ is Lipschitz continuous with respect to Minkowski distances.
\end{proposition}
Given these elements, the following proposition describes the tail of a univariate output of a Standard VAE.
\begin{proposition}\label{thm: VAE margins}
    For the Standard VAE of Example \ref{ex: standard VAE} with univariate output ($m=1$), given that the neural network functions $\mu_\theta$ and $\sigma_\theta$ of the probabilistic decoder are piecewise linear operators, then the output distribution is light-tailed.
\end{proposition}

\begin{corollary}
    The marginal distributions generated by the Standard VAE of Example \ref{ex: standard VAE} are light-tailed, whenever the neural networks functions $\mu_\theta$ and $\sigma_\theta$ are composed of operations described in Proposition \ref{prop: PWL NN}.
\end{corollary}

\subsection{Angular Measure of ReLu Networks Transformation of Random Vectors}\label{limit GANs}
In this section, we focus on the angular measures associated with generative models. We demonstrate that distributions sampled by algorithms based on the transformation of a random input vector through a ReLU neural network (a feed-forward network with linear layers and ReLU activation functions) have angular measures concentrated on a finite set of points on the simplex. Although not specific to VAEs, this result suggests a particular focus on the representation of the angular measure in the VAE framework.
\\
Let us consider the following framework for generating multivariate heavy-tailed data in the non-negative orthant:
    \begin{equation}\label{eq: GANs}
        {\bf X}=f({\bf Z}),
    \end{equation}
with ${\bf Z}$ an $n$-dimensional input random vector with i.i.d heavy-tailed marginals and f a ReLU neural network with output in $\mathbb{R}_+^m$.
This heavy-tailed sampling framework is used by \cite{feder2020nonlinear} and \cite{huster2021pareto} for GANs, and \cite{jaini2020tails} for NFs. Within this framework ${\bf X}$ is heavy-tailed with the same shape parameter as ${\bf Z}$, while it would be light-tailed if ${\bf Z}$ were light-tailed .\\ 
In the limit of extreme values, one may ask which dependence structures between the marginals of ${\bf X}$ can be represented by such a model. This corresponds to the angular measure defined in Equation \eqref{eq: wan MRV}. If we designate ${\bf S}_{\bf X}$ as the angular measure of ${\bf X}$, we can state the following Proposition.
\begin{proposition}\label{thm: GAN and limit distribution}
  Under the framework described in Equation \eqref{eq: GANs}, ${\bf S}_{\bf X}$ is concentrated on a finite set of points of the simplex less than $n$. In other words, it means that there exist some vectors $\{{\bf v}_1, {\bf v}_2, \cdots, {\bf v}_{n^\prime}\}$ with $n^\prime\leq n$ such that for any subset $\mathbb{A}$ of the $(m-1)$-dimensional simplex
  \begin{equation*}
      {\bf S}_{\bf X}(\mathbb{A}) = \sum_{i=1}^{n^\prime} p_i \delta_{{\bf v}_i}(\mathbb{A}),
  \end{equation*}
where $\delta$ is the Dirac measure and $p_i>0$ such that $\sum_{i=1}^{n^\prime} p_i = 1$.
\end{proposition}
Informally, this proposition means that for sufficiently large radii, ${\bf X}$ tends to concentrate along one of a finite set of direction vectors. While extracting certain principal directions in extreme regions is a useful tool for the comprehensive analysis of a data set \citep{drees2021principal}, it is severely lacking in flexibility when it comes to represent more complex distributions and generate realistic extreme samples.\\
To circumvent this difficulty, we consider a polar decomposition of $\mathbf{X}$ into radial and angular components, which are generated separately. This allows us to obtain angular measures that are more diverse than those concentrated on a finite set of vectors, as illustrated by the numerical experiments reported in Section \ref{experiments}. 

\section{Proposed VAE Architecture}\label{vif}
\subsection{Overall strategy}
Our approach aims to adapt the VAE framework in order to generate multivariate regularly varying random vectors. To do so, the key principle of our approach is to separate the generation of the radius and of the angle. Following the polar decomposition $\mathbf{X} = R \mathbf{\Theta}$, we first aim to sample a heavy-tailed radius $R$, and then, conditionally on the sampled radius, to generate an angle $\mathbf{\Theta}$ representing the dependence structure between the marginals. Thus we rely on two VAEs, the first one to sample the radius, the other one to sample the angle conditionally on the radius.\\

This separation allows us to model both the marginal tail behavior and the angular dependence of extremes in a flexible way, consistent with the theory of multivariate regular variation (Definition~\ref{def: MRV}).\\ 

A major challenge lies in sampling a heavy-tailed radius. 
As shown in Proposition~\ref{thm: VAE margins}, the classical Gaussian parameterization of the VAE cannot generate heavy-tailed distributions. 
To overcome this limitation, we propose in Section~\ref{radius} a new parameterization ensuring that the generated radius follows a heavy-tailed distribution. 
To provide intuition on how this parameterization works, Section~\ref{multrad} introduces a simplified multiplicative framework in which the resulting distributions are provably heavy-tailed (see Proposition~\ref{prop: breiman}).\\ 

Additionally, our aim is to represent rich and realistic angular measures on the simplex, avoiding their concentration on a few points, as is typically the case in the framework described in Proposition~\ref{thm: GAN and limit distribution}. 
This flexibility is naturally induced by the polar decomposition and is based on a conditional VAE that models the conditional distribution $\mathbf{\Theta}\mid R$ (Section~\ref{sphere}). 
In this conditional framework, we enforce asymptotic independence between the radius and the angle for large values of the radius, as required by multivariate regular variation, while enabling a diverse and smooth representation of the angular dependence at extreme levels.\\

To summarize, our strategy to generate a sample $\mathbf{x}^{(i)}$ of a multivariate regularly varying random vector is described by the following three-step VAE scheme:
\begin{itemize}
    \item Using a VAE, a radius $r^{(i)}$ is drawn from a univariate heavy-tailed distribution $R$ (see Section \ref{radius});
    \item Conditionally on the drawn radius $r^{(i)}$, we sample $\mathbf{\Theta}^{(i)}$ an element of the $(m-1)$-dimensional simplex from the conditional distribution $\mathbf{\Theta}\mid[R=r^{(i)}]$ while forcing the independence between radius $R$ and angle $\mathbf{\Theta}$ for large values of the radius. We use a conditional VAE for this purpose (see Section \ref{sphere});
    \item We multiply componentwise the angle vector by the radius to obtain the desired sample, i.e.\ $\mathbf{x}^{(i)} = r^{(i)} \mathbf{\Theta}^{(i)}$.
\end{itemize}
This generative strategy is illustrated in Figure \ref{fig: architecture}. 
\begin{figure}[h!]
\centering
\begin{tikzpicture}[
font=\small,
node distance=10mm and 14mm,
>=Latex,
block/.style={draw, rounded corners, inner sep=6pt, align=center},
mul/.style={draw, circle, minimum size=7mm, inner sep=0pt},
group/.style={draw, rounded corners, inner sep=6pt, fill=gray!40, fill opacity=0.2},
]

\node[block, fill=red!30] (b1) {Sample $z_{\mathrm{rad}}^{(i)}$\\ from $p(z_{\mathrm{rad}})$}; 
\node[block, fill=blue!30, below=of b1] (b2) {Sample $r^{(i)}$\\ from $p\left(r \mid z_{\mathrm{rad}}^{(i)}\right)$}; 
\draw[->] (b1) -- (b2);

\node[block, fill=red!30, right=32mm of b1] (b3) {Sample $\mathbf{z}_{\mathrm{ang}}^{(i)}$\\ from $p(\mathbf{z}_{\mathrm{ang}})$}; 
\node[block,fill=blue!30, below=of b3] (b4) {Sample $\mathbf{\Theta}^{(i)}$\\ from $p\left(\mathbf{\Theta} \mid [\mathbf{z}_{\mathrm{ang}}^{(i)},\; r^{(i)}]\right)$}; 
\draw[->] (b3) -- (b4);

\draw[->] (b2.east) -- ($ (b4.west) + (-8mm,0) $) -- (b4.west);

\node[mul, below=18mm of $(b2)!0.5!(b4)$] (b5) {$\times$}; 
\draw[->] (b2) |- (b5);
\draw[->] (b4) |- (b5);

\node[block, below=10mm of b5,fill=green!30] (b6) {$\mathbf{x}^{(i)}= r^{(i)}\mathbf{\Theta}^{(i)}$};
\draw[->] (b5) -- (b6);

\node[group, fit=(b1) (b2), label={[yshift=2mm]above:{\normalsize VAE for $R$}}] (gR) {};
\node[group, fit=(b3) (b4), label={[yshift=2mm]above:{\normalsize Conditional VAE for $\mathbf{\Theta}\mid R$}}] (gT) {};

\end{tikzpicture}

\caption{VAE scheme to draw a sample $\mathbf{x}^{(i)}$ (green block) from a multivariate regularly varying random vector. Two VAEs are involved (grey areas): one for radius generation (left) and one for angle generation (right). Each VAE relies on its own latent prior distribution (red blocks) and conditional distribution (blue blocks). Arrows indicate causal relations; the arrow between blue blocks shows that the angle is sampled conditionally on the radius.}
\label{fig: architecture}
\end{figure}
\subsection{Intuition Behind the Heavy-tailed Radius Sampling Scheme}
\label{multrad}
Before detailing the specific parameterization that enables us to generate a heavy-tailed radius within a VAE framework, we first introduce a simplified multiplicative framework that guided our design choices and directly inspired the final formulation. The idea is to model the radius $R$ through a heavy-tailed latent variable $Z_{rad}$, and to apply to it a very simple transformation: namely, a multiplication by a lighter-tailed independent random variable. More precisely, we consider the following two conditions which, when satisfied, guarantee that the resulting radius is heavy-tailed.

\begin{condition}\label{Xinvgamma}
    $Z_{rad}$ follows an inverse-gamma distribution defined for $z_{rad}>0$ by the pdf
\begin{equation}\label{eq: ig}
{\displaystyle f_{\bf Inv\Gamma}(z_{rad}\ ;\ \alpha ,\beta )={\frac {\beta ^{\alpha }}{\Gamma (\alpha )}}z_{rad}^{-\alpha -1}\exp \left(-\beta /z_{rad}\right)},
\end{equation}
with $\alpha$ and $\beta$ two strictly positive constants.
\end{condition}
\begin{condition}\label{Ymult}
    $R$ is linked to $Z_{rad}$ throughout a multiplicative model with a positive random coefficient $A$, i.e. 
    \begin{equation*}
            R = A \times Z_{rad}, 
    \end{equation*}
     where the random variable $A$ is absolutely continuous and independent of $Z_{rad}$. We also assume that $0<E\left[A^{\alpha+\epsilon}\right]<\infty$ for some positive $\epsilon$.
\end{condition}
We recall that the inverse-gamma distribution is heavy-tailed with tail index $\alpha$ and has a strictly positive support. Above, the moment condition $0<E\left[A^{\alpha+\epsilon}\right]<\infty$ means that $A$ has a significantly lighter tail than $Z_{rad}$.\\
Under these two conditions, Breiman's lemma \citep{breiman1965some} guarantees that the parameterization considered in Condition \ref{Ymult} leads to a heavy-tailed distribution of the radius $R$. Formally, the following proposition holds.
\begin{proposition}\label{prop: breiman}
If Conditions \ref{Xinvgamma} and \ref{Ymult} hold, $R$ is heavy-tailed with tail index $\alpha$. In particular, if $A$ follows an exponential distribution with scale parameter $c$ then $R$ follows a GP distribution (see Equation \ref{eq: GP} with $\xi =\frac{1}{\alpha}$ and $\sigma = \frac{\beta c}{\alpha}$) 
\end{proposition}


\subsection{Sampling from Heavy-tailed Radius Distributions}\label{radius}
The first component of our framework is a VAE designed to sample a heavy-tailed radius $R$.
To tailor the VAE framework introduced in Section \ref{amortize VI} to the sampling of heavy-tailed random variables, we set the prior $Z_{rad}$ to follow inverse gamma distribution with parameters $\alpha$ and $\beta$, i.e. $Z_{rad}$ satisfies Condition \ref{Xinvgamma}. Notice that if $Z_{rad}$ follows an inverse gamma distribution with parameters $\alpha$ and $\beta$, then for each $c>0$, $cZ_{rad}$ is an inverse gamma with parameters $\alpha$ and $c\beta$. Consequently, and without loss of generality, we set parameter $\beta$ of $Z_{rad}$ equal to $1$. Overall, we replace the light-tailed system described in Example \ref{ex: standard VAE} by the following heavy-tailed system:%
\begin{align}
p_{\alpha}(z_{rad})&= f_{\bf Inv\Gamma}(z_{rad}\ ;\ \alpha ,1),\label{eq: rad_prior}\\
p_\theta(r\mid z_{rad})&=f_{\mathbf{\Gamma}}\left(r\ ;\ \alpha_\theta(z_{rad}),\beta_\theta(z_{rad})\right), \label{eq: Gamma dist}\\    
q_\phi(z_{rad}\mid r)&=f_{\mathbf{Inv\Gamma}}\left(z_{rad}\ ;\ \alpha_\phi(r),\beta_\phi(r)\right),\label{eq: target_rad} 
\end{align}
with $f_{\mathbf{\Gamma}}$ the pdf of a Gamma distribution. $\alpha_\theta$, $\beta_\theta$, $\alpha_\phi$, $\beta_\phi$ are ReLU neural networks functions with parameters $\theta$ and $\phi$. We may stress that the above parameterization ensures the non-negativeness of the samples both for the target and the likelihood.\\
The following proposition characterizes the heavy-tailed behavior of this univariate VAE scheme.  
\begin{proposition}\label{prop: Radii VAE}
We consider the VAE parameterization described by Equations \eqref{eq: rad_prior}, \eqref{eq: Gamma dist} and \eqref{eq: target_rad}.
If we further assume that the function $\alpha_{\theta}(.)$ is equal to a strictly positive constant, and the function $\beta_{\theta}(.)$ satisfies
\begin{align}
    \lim_{z_{rad} \to +\infty} \beta_{\theta}(z_{rad}) \propto \frac{1}{z_{rad}},\label{eq: limit beta}\\
    \lim_{z_{rad} \to 0} \beta_{\theta}(z_{rad}) \propto \frac{1}{z_{rad}}, \label{eq: limit beta 2},
\end{align}
then the univariate output distribution sampled by this VAE scheme is heavy-tailed with tail index equal to $\alpha$.
\end{proposition}
In our implementation, we impose on $\beta_\theta(.)$ to satisfy Equations \eqref{eq: limit beta} and \eqref{eq: limit beta 2} by choosing
\begin{equation}\label{eq: f_theta}
\beta_\theta(z_{rad}) = \frac{|f_\theta(z_{rad})|}{z_{rad}^2},    
\end{equation}
where $f_\theta$ is a ReLU neural network. Additionally, we allow $\alpha_\theta(\cdot)$ to be more flexible than a constant function, imposing only that it admits a strictly positive limit at infinity. This corresponds to
\begin{equation}\label{eq: g_theta}
    \alpha_\theta(z_{rad}) = \frac{|g_\theta(z_{rad})|}{z_{rad}},
\end{equation}
where, again, $g_\theta$ is a ReLU neural network.\\
We choose our parameterization based on an analogy with the ideal multiplicative framework described in Section 4.1. Indeed, considering Conditions \ref{Xinvgamma} and \ref{Ymult} verified, then
\begin{align}
    R\mid [Z_{rad}&=z_{rad}] \overset{\mathrm{d}}{=} Az_{rad},\\
    Z_{rad} \mid [R&=r]\overset{\mathrm{d}}{=} \frac{r}{A}.\label{eq: cond Z_rad}
\end{align}
Since $A$ must have a lighter tail than $Z_{rad}$ to satisfy the moment condition 
$E\left[ A^{\alpha +\epsilon} \right] < \infty$ for some positive $\epsilon$, we choose the approximate likelihood $p_\theta(r \mid z_{rad})$ in a light-tailed distribution family (i.e. Gamma distribution). Considering Equation \eqref{eq: cond Z_rad}, we notice that $Z_{rad}\mid [R=r]$ could be heavy-tailed if $A$ have non-null probability on each open set containing $0$. Thus we choose a heavy-tailed distribution for the target (i.e. Inverse-Gamma distribution). Additionally, as $R$ and $Z_{rad}$ are positive random variables, our parameterization ensures that negative values for either target and likelihood cannot occur.\\

Besides, by parameterizing both the prior $p_\alpha$ and the target $q_\phi$ as Inverse-Gamma distributions, the KL divergence in Equation \eqref{eq: ELBO_VAE} admits an analytical expression.
\begin{proposition}\label{prop: DKL InvGamma}
Given expression \eqref{eq: rad_prior} and \eqref{eq: target_rad} for prior and target distributions, the KL divergence in Equation \eqref{eq: ELBO_VAE} is given by 
\begin{eqnarray}\label{eq: DKL analytic}
    D_{KL}\left(q_\phi(z_{rad}\mid r)\mid \mid p_\alpha(z_{rad})\right) = 
    (\alpha_\phi(r) - \alpha)\psi(\alpha) - \log\frac{\Gamma(\alpha_\phi(r))}{\Gamma(\alpha)} \nonumber \\
    +\alpha\log\beta_\phi(r) + \alpha_\phi(r)\frac{1-\beta_\phi(r)}{\beta_\phi(r)}, 
\end{eqnarray}
where $\Gamma$ and $\psi$ stands respectively for the gamma and digamma functions.
\end{proposition}
Interestingly, this proposition provides the basis for learning tail index $\alpha$ from data, which is a challenging issue in EVT \citep[see, e.g.][]{danielsson2016tail}. \\

\subsection{Conditional Sampling of the Angle given the Radius}\label{sphere}
The second component of our overall framework is a conditional VAE \citep[see, e.g.,][]{zhao2017learning} which generates the angle given the previously sampled radius. More formally, it draws samples from the conditional distribution ${\bf \Theta}\mid R$. This angular VAE has a multivariate normal prior ${\bf Z}_{ang}$. The target is also parameterized as a multivariate normal distribution, whose mean and standard deviation depend on both the latent variable and the radius. The likelihood is parameterized as a projection of a normal distribution on the $\mathcal{L}_1$ sphere. This projection is denoted  $\bf \Pi$, and is such that for any vector $\bf s$, 
$$
\bf \Pi({\bf s}) = \frac{{\bf s}}{\|{\bf s}\|},
$$
where the considered norm is the $\mathcal{L}_1$-norm. Additionally, we define ${\bf S}({\bf\Theta})=\{{\bf s}, \ {\bf \Pi}({\bf s}) = {\bf \Theta}\}$. Overall, our conditional angular VAE relies on the following parameterization:
\begin{align}\label{eq: ang param}
     p({\bf z}_{ang})& = \mathcal{N}({\bf z}_{ang} \  ; {\bf 0},{\bf I}_n),\nonumber \\
    p_\nu({\bf \Theta}\mid {\bf z}_{ang},r)&= \int_{{\bf S}({\bf\Theta})}\mathcal{N}\left({\bf s}\ ; \ {\bf \mu}_{\nu}({\bf z}_{ang},r),\operatorname{diag}({\bf \sigma}_\nu({\bf z}_{ang},r))^2\right),\\
    q_\omega({\bf z}_{ang}\mid {\bf \Theta},r)&=\mathcal{N}\left({\bf z}_{ang} \ ; \ {\bf \mu}_{\omega}({\bf \Theta},r),\operatorname{diag}({\bf \sigma}_{\omega}({\bf \Theta},r))^2\right),\nonumber
\end{align}
 where $n$ is the dimension of the latent space, ${\bf \mu}_{\nu}$, ${\bf \sigma}_\nu$, ${\bf \mu}_{\omega}$ and ${\bf \sigma}_{\omega}$ are neural network functions with parameters $\nu$ and $\omega$. The dependence on $R$ has been made explicit in both the target and the likelihood, thereby rendering the framework conditional.\\
As our initial aim is to sample on the multivariate simplex rather than on the multivariate sphere, we also use a Dirichlet parameterization of the likelihood. Details regarding this parameterization can be found in Appendix \ref{app: Dirichlet}.\\
To sample from multivariate regularly varying random vectors (Definition \ref{def: MRV}), we enforce the independence between the respective distributions of the radius and the angle when $r\to +\infty$ (see Equation \eqref{eq: wan MRV}). To do so, we ensure that the functions ${\bf \mu}_{\nu}$ and ${\bf \sigma}_{\nu}$ satisfy the following necessary condition.
\begin{condition}\label{cond: indep ang rad}
    ${\bf \mu}_{\nu}$ and ${\bf \sigma}_{\nu}$ are such that there exist two $z$-varying functions ${\bf\mu}_{\infty}$ and ${\bf \sigma}_{\infty}$ which verify for each ${\bf z}_{ang}$
    \begin{align*}
        \lim_{r \to +\infty}{\bf \mu}_{\nu}({\bf z}_{ang},r) &= {\bf\mu}_{\infty}(\mathbf{z}_{ang}),\\
        \lim_{r \to +\infty}{\bf \sigma}_{\nu}({\bf z}_{ang},r) &= {\bf\sigma}_{\infty}({\bf z}_{ang}).\\
    \end{align*}
\end{condition}
In our implementation, Condition \ref{cond: indep ang rad} is enforced by imposing 
\begin{align}
    {\bf \mu}_{\nu}({\bf z}_{ang},r) &= f_{\nu}\left({\bf z}_{ang},\frac{1}{1+r}\right), \label{eq: mu_inf}\\
    {\bf \sigma}_{\nu}({\bf z}_{ang},r) &= g_{\nu}\left({\bf z}_{ang},\frac{1}{1+r}\right)\label{eq: sig_inf},
\end{align}
with $f_{\nu}$ and $g_\nu$ Lipschitz continuous neural networks.
\begin{remark}\label{rk: param_inf}
    From Equations \eqref{eq: mu_inf} and \eqref{eq: sig_inf}, we deduce
    \begin{align*}
         {\bf \mu}_{\infty}(\mathbf{z}_{ang}) = f_{\nu}({\bf z}_{ang},0),\\
         {\bf\sigma}_{\infty}({\bf z}_{ang})  = g_{\nu}({\bf z}_{ang},0).
    \end{align*}
    Thus, sampling from the angular measure is an easy task as it is enough to: (i) draw sample ${\bf z}_{ang}$ from the prior $\mathcal{N}(0,{\bf I}_{n})$, (ii) sample from $\mathcal{N}\left({\bf \mu}_{\infty}(\mathbf{z}_{ang}),\operatorname{diag}({\bf\sigma}_{\infty}({\bf z}_{ang}))^2\right)$, and (iii) project onto the $\mathcal{L}_1$ sphere through ${\bf \Pi}$.
    \end{remark}
    
\section{Implementation}\label{lf}
This section introduces several implementation aspects of our approach. 
We first describe the architecture of the trained VAEs in Section~\ref{training settings}, and then outline the learning set-up in Section~\ref{learning setup}. We also introduce in Section \ref{perf assess} performance metrics used for benchmarking purposes, and in Section \ref{notations} the approaches with which we compare the proposed VAE scheme.  
 
\subsection{Neural network parameterizations}\label{training settings}
In this section, we describe in detail the chosen parameterization of the neural architectures for the two VAEs introduced in Section~\ref{vif}, as used in the numerical experiments.\\ 
For the radius generation VAE described in Section \ref{radius}, we consider ReLU neural networks with the following parameterizations:
\begin{itemize}
\item For the probabilistic encoder, we set two 5-dimensional hidden layers with ReLU activation. The output layer is a 2-dimensional dense layer with ReLU activation. For convergence purposes, we initialize the weights of the dense layers to 0  and their biases to a strictly positive value sampled from a uniform distribution between 1 and 2.
	\item For the probabilistic decoder, we detail the architecture of $f_\theta$ and $g_\theta$ of Equations \eqref{eq: f_theta} and \eqref{eq: g_theta}. We consider the same architecture as the probabilistic encoder. Regarding the output, one corresponds to the output $f_\theta$ and the other one to the output of $g_\theta$. The output bias of $f_\theta$ is initialized as a strictly positive value (randomly sampled from a uniform distribution between 1 and 2) and the output kernel of $g_\theta$ is initialized as a positive value (randomly sampled from a uniform distribution between 0.1 and 2).
\end{itemize}
For the angular VAE described in Section \ref{sphere}, the following parameterization are considered:
\begin{itemize}
\item For the encoder, the latent dimension is 4. We consider 3 hidden layers with ReLU activation, respectively with 8, 8 and 4 output features. The output layer is a dense linear layer. We use the default initialization scheme.
\item For the decoder, the input radius is first transformed according to Equations \eqref{eq: mu_inf} and \eqref{eq: sig_inf}. We use 3 hidden layers with ReLU activation, respectively with 5, 10 and 5 output features. 
The output layer is a dense layer. We use the default initialization scheme, except for the bias of the final layer, which is initially sampled from a uniform distribution between 0.5 and 3.
\end{itemize}
\subsection{Learning Set-up}\label{learning setup}
The considered training procedure follows from our hierarchical architecture with two VAEs and involves two distinct training losses, denoted by $\mathcal{L}_{R}$ for the training loss of the radius VAE and $\mathcal{L}_{{\bf \Theta}\mid R}$ for that of the angular VAE. 
For a data set $(\mathbf{x}^{(i)})_{i=1}^{N}$ with polar decomposition $\left(r^{(i)}, {\bf \Theta}^{(i)}\right)$, we derive the training loss $\mathcal{L}_{R}$ from Equations \eqref{eq: ELBO_VAE} and \eqref{eq: DKL analytic} as
\begin{align*}
    \mathcal{L}_{R}(\alpha,\theta,\phi) = &\sum_{i=1}^N \Bigg[\bigg((\alpha_\phi(r^{(i)}) - \alpha)\psi(\alpha) - \log\frac{\Gamma(\alpha_\phi(r^{(i)}))}{\Gamma(\alpha)} +\alpha\log\beta_\phi(r^{(i)}) + \alpha_\phi(r^{(i)})\frac{1-\beta_\phi(r^{(i)})}{\beta_\phi(r^{(i)})}\bigg)\\
    &+\frac{1}{L}\sum_{l=1}^L \log f_{\mathbf{\Gamma}}\left(r^{(i)}\ ;\ \alpha_\theta(z_{rad}^{(i,l)}),\beta_\theta(z_{rad}^{(i,l)})\right) \Bigg],
\end{align*}
Similarly, the training loss $\mathcal{L}_{{\bf \Theta}\mid R}$ writes as
\begin{align*}
    \mathcal{L}_{{\bf \Theta}\mid R}(\nu, \omega) = &\sum_{i=1}^N \Bigg[ \frac{1}{2}\sum_{j=1}^n\Big(1 +2\log[\sigma_\omega({\bf \Theta}^{(i)},r^{(i)})]_j - [\mu_\omega({\bf \Theta}^{(i)},r^{(i)})]_j^2 - [\sigma_\omega({\bf \Theta}^{(i)},r^{(i)})]_j^2\Big)\\
    &+\frac{1}{L}\sum_{l=1}^L \log \mathcal{N}\left({\bf \Theta}^{(i)}\ ; \ {\bf \mu}_{\nu}({\bf z}_{ang}^{(i,l)},r^{(i)}),\; \operatorname{diag}({\bf \sigma}_\nu({\bf z}_{ang}^{(i,l)},r^{(i)})^2\right)\Bigg].
\end{align*}
In the implementation of these training losses, we sample each $z_{rad}^{(i,l)}$ from the pdf $q_\phi(z_{rad}\mid r^{(i)})$, and each ${\bf z}_{ang}^{(i,l)}$ from the pdf $q_\omega({\bf z}_{ang}\mid{\bf \Theta}^{(i)},r^{(i)})$. Overall, our training loss $\mathcal{L}_{ExtVAE}$ is the sum
\begin{equation}\label{eq: training loss}
    \mathcal{L}_{ExtVAE} (\alpha,\theta,\phi,\nu, \omega)= \mathcal{L}_R (\alpha,\theta,\phi) + \mathcal{L}_{{\bf \Theta}\mid R}(\nu, \omega). 
\end{equation}
In practice, we first train the radius VAE, {\em i.e.} parameters $(\alpha,\theta,\phi)$, and second the angular VAE, {\em i.e.} parameters $(\nu, \omega)$.
Depending on the experiments, the parameter $\alpha$ of the radius prior can either be known or unknown. When known, it suffices to set $\alpha$ equal to the desired value in Equation \eqref{eq: training loss}. When unknown, $\alpha$ can be directly optimized by gradient descent.\\ 
For estimating $(\alpha,\theta,\phi)$, the training is limited to 5000 epochs, and the learning rate set to $10^{-4}$.
The same maximum number of epochs is used to estimate $(\nu, \omega)$ but the learning rate is fixed to $10^{-5}$.\\
In both cases, we used Adam optimizer \citep{kingma2014adam} and a batch size of $32$. From a code perspective, we made extensive use of the Tensorflow and Tensorflow-Probability libraries. The whole code is freely available.\footnote{The implementation is available at \url{https://github.com/Nicolasecl16/ExtVAE}.}

\subsection{Performance Assessment}\label{perf assess}
We present the various criteria used to evaluate the different approaches tested in our numerical experiments. These criteria can be grouped into three categories, depending on whether they relate to radius distributions, output distributions or angular measures.\\
For the radius distribution, log-quantile-quantile plots \citep[for detailed examples, see][Chapter 4]{resnick2007heavy}, abbreviated as log-QQ plots, are graphical methods we use to informally assess the goodness-of-fit of our model to data. This method consists in plotting the log of the empirical quantiles of a generated sample against those of the experimental data. If the fit is good, the plot should be roughly linear. We use the ELBO cost (Equation \ref{eq: ELBO_VAE}) on a given data set as a numerical indicator to compare the radius distribution obtained with our VAE approach to a vanilla VAE not tailored for extremes. Another criterion considered is an estimator of the KL divergence between the data distribution and that generated by the tested approaches, together with a variant proposed by \cite{naveau2014non}. This variant estimates the KL divergence above a specified threshold (see Appendix~\ref{app: KLu}).\\ 
For the whole set of generated samples, we investigate several other criteria. We compute the Wasserstein distances between extreme samples generated by the tested approaches and extreme true data samples.
To do so, we select a large threshold $u$ and compute the Wasserstein distance above this threshold by restricting the samples to points with a radius greater than $u$. Finally, we rescale this distance by dividing it by $u^2$ (see Appendix~\ref{app: WD}). To compute the Wasserstein distances, we use pre-implemented functions from the Python Optimal Transport package \citep[see][]{flamary2021pot}.\footnote{The documentation is available at \url{ https://pythonot.github.io/quickstart.html}}\\
We have seen that for a multivariate regularly varying random vector, the radius and the angle becomes independent when the radius tends to infinity (see Equation \ref{eq: wan MRV}). In practice, one can consider the radius and the angle independent by choosing a sufficiently large radius threshold. \cite{wan2019threshold} have established a criterion to detect whether the respective distributions of the radius and the angle can be considered as independent, and thus to choose the corresponding threshold radius. This allows us to compare the threshold radii between the true data and the generated data. We rely on the testing framework introduced in \cite{wan2019threshold} to calculate a $p$-value that follows a uniform distribution if the distributions of the radius and the angle are independent, and that is close to 0 otherwise (see Appendix \ref{app: Threshold selection}).

\subsection{Notations and Benchmarked Approaches}\label{notations}
We refer to our generative approach as ExtVAE if we assume that the tail index $\alpha$ is known, and as UExtVAE if the tail index is learned from data. If we restrict ourselves to the radii generated by ExtVAE and UExtVAE via the procedure described in Section \ref{radius}, we denote respectively ExtVAE$_r$ and UExtVAE$_r$. We compare our approach with the Standard VAE of Example \ref{ex: standard VAE}, i.e. with normal distribution for prior, target and likelihood, indicated by the acronym StdVAE. We also compare our approach with ParetoGAN which is the GAN scheme for generating extremes proposed by \cite{huster2021pareto}. The ParetoGAN is a Wasserstein GAN \citep[see][]{arjovsky2017wasserstein} with Pareto prior. Given the difficulty of training a GAN, as well as the number of factors that can influence the results it produces, we empirically tuned the ParetoGAN architecture to provide a sensible GAN baseline in our experiments. Although this parameterization may not be optimal, our interest goes beyond a simple quantitative intercomparison 
in exploring and understanding the differences between the proposed VAE approach and GANs in their ability to represent and sample extremes.

\section{Experiments}\label{experiments}
We conduct experiments on synthetic and real multivariate data sets. The synthetic data set involves a heavy-tailed radius distribution and the angular distribution on the multivariate simplex is a Dirichlet distribution with radius-dependent parameters. The real data set corresponds to a monitoring of Danube river network discharges. 

\subsection{Synthetic Data Set}

We first consider a synthetic data set of samples from a 5-dimensional heavy-tailed random vector with a tail index $\alpha=1.5$. We detail the simulation setting in Appendix \ref{app: synthesized dataset}. 
The training data set consists of 250 samples, compared to 750 for the validation data set and 10000 for the test data set.\\
\begin{table}
\caption{\label{tab: R1 loss} Mean ELBO cost (see Equation \ref{eq: ELBO_VAE}) computed on the training (Train), validation (Val) and test (Test) data set for the radius variable $R_1$.}
\begin{center}
\begin{tabular}{|c|c|c|c|}
  \hline
  Approach & Train loss & Val loss  & Test loss  \\
  \hline
  StdVAE & $1.21$ & $4.81$ &  $+\infty$\\
  \hline
  ExtVAE$_r$ & $0.88$ & $1.10$  & $1.12$\\
  \hline
  UExtVAE$_r$ & $0.95$ & $1.12$  & $1.15$\\
  \hline
\end{tabular}
\end{center}
\end{table}
\begin{figure}
\centering
    \includegraphics[width=0.6\columnwidth]{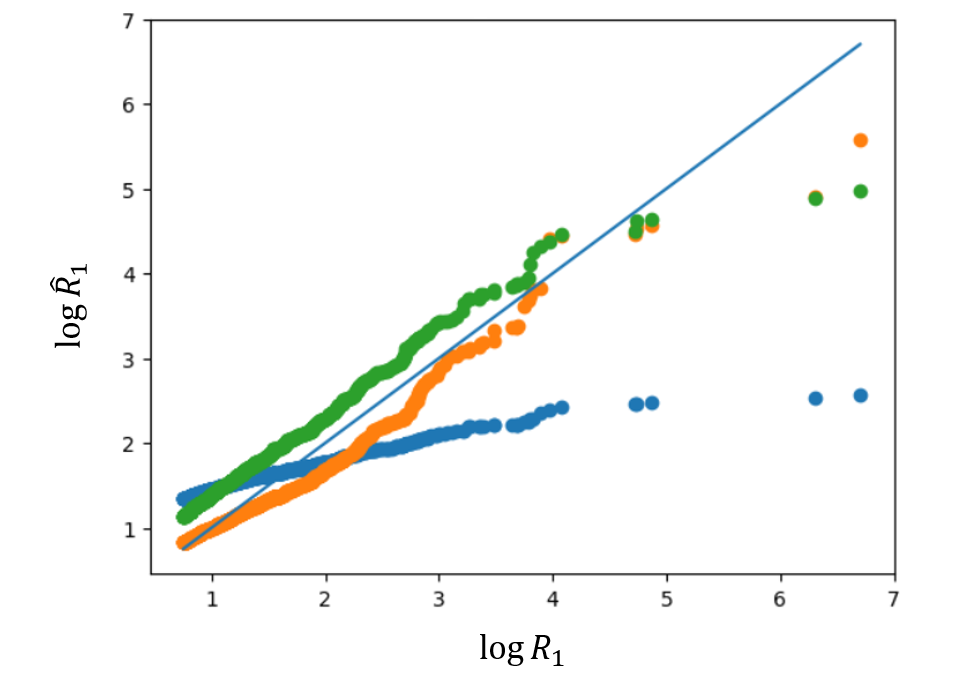}
    \label{fig: log QQplot UExtVAE}
\caption{Log-QQ plot between the upper decile of 10000 radii samples from StdVAE (blue dots), ExtVAE$_r$ (orange dots), UExtVAE$_r$ (green dots) and the upper decile of the test data set of $R_1$. The log values of the true radius, denoted $\log R_1$ is on the x-axis, the log of the estimated radius, denoted $\log \hat{R}_1$, is on the y-axis. The dots should lie close to the blue line}
\label{fig: QQ plots}
\end{figure}
In Table \ref{tab: R1 loss} and Figure \ref{fig: QQ plots}, we study the ability of the benchmarked VAE schemes to sample heavy-tailed radius distribution. 
The results in Table~\ref{tab: R1 loss} indicate that the mean ELBO cost remains approximately constant for our approaches across the training, validation, and test data sets, whereas it diverges for StdVAE. This indicates that our approaches, unlike StdVAE, successfully extrapolate the tail of the radius distribution. The log-QQ plots given in Figure \ref{fig: QQ plots} illustrate further that ExtVAE$_r$ and UExtVAE$_r$ schemes relevantly reproduce the tail pattern of the radius distribution $R_1$ while this is not the case for StdVAE. 
Figure \ref{fig: KL upon u} represents, for the compared methods, the evolution of the KL divergence between the true distribution and the simulated ones above a varying quantile $u$ (Equation \ref{eq: KL div threshold}). Again, the StdVAE poorly matches the target distribution with a clear increasing trend for quantiles $u$ such that $P(R_1>u)\geq0.3$. 
Conversely, the KL divergence is much smaller and much more stable for ExtVAE$_r$ and UExtVAE$_r$ schemes, especially for large quantile values.
\begin{figure}
\centering
    \includegraphics[width=0.6\columnwidth]{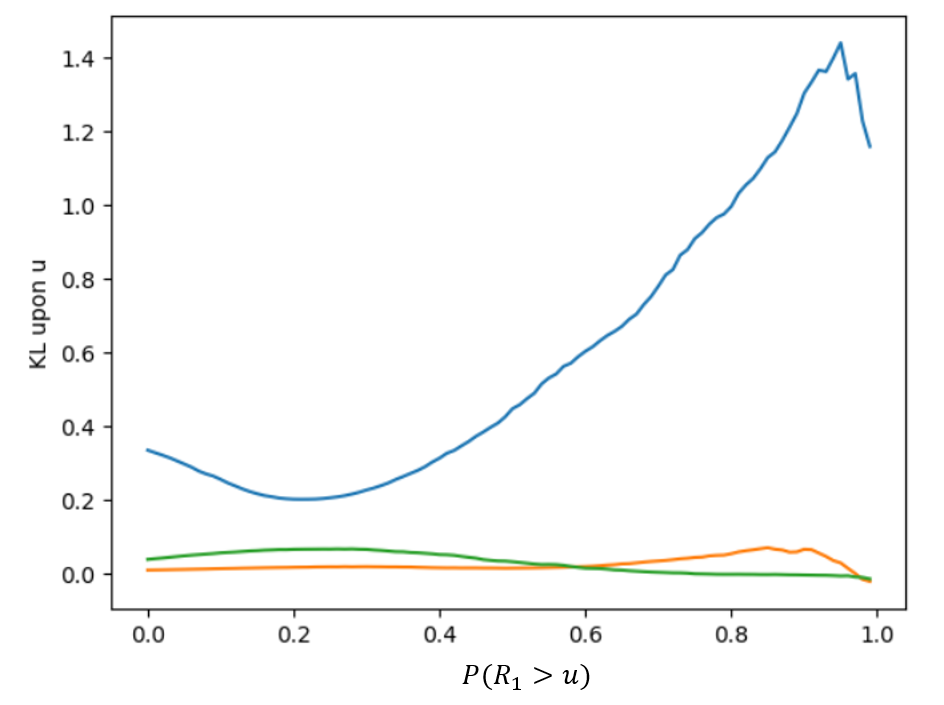}
\caption{KL divergence between the radius distribution of the benchmarked VAE models and the target heavy-tailed distribution: we display the KL divergence (see Equation \ref{eq: KL div threshold}) above quantile $u$ for $P(R_1>u)$ varying from $0$ to $1$. The compared VAEs are the StdVAE (blue curve), the ExtVAE$_r$ (orange curve) and the UExtVAE$_r$ (green curve). To estimate the KL divergences, 10,000 samples from each distribution are drawn, and $u$ is set to the quantile computed from the samples of $R_1$.}
\label{fig: KL upon u}
\end{figure}
Interestingly, for the different criteria, the results obtained with UExtVAE$_r$ are very close or even indistinguishable from those obtained with ExtVAE$_r$. This suggests that the estimation of the tail index is accurate.
In order to better assess the robustness of this estimation, we report in Figure \ref{fig: Learned tail index} the evolution of the tail index of UExtVAE$_r$ as a function of the training epochs for randomly chosen initial values.
\begin{figure}[ht]
    \centering
    \includegraphics[width=0.6\columnwidth]{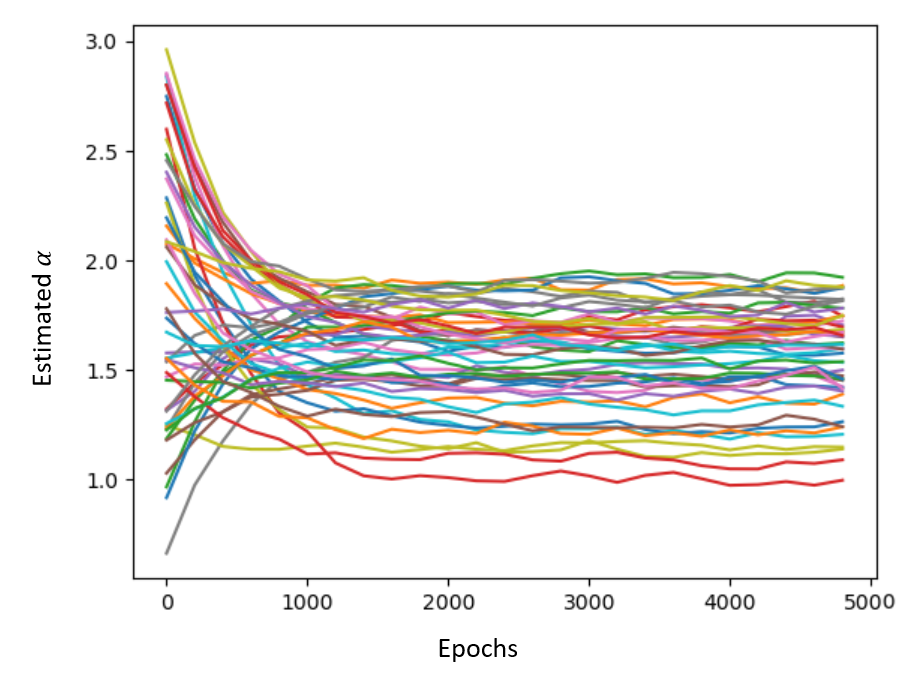}
    \caption{Evolution of the tail index $\alpha$ of UExtVAE$_r$ during the training procedure: we report the value of the tail index as a function of the training epochs for training runs from different initial values. The initial values of $\alpha$ are sampled uniformly between $0.5$ to $3$. The true value of $\alpha$ is $1.5$.}
    \label{fig: Learned tail index}
\end{figure}
Given the expected uncertainty in estimating the tail index (see Appendix \ref{app: tail index}), UExtVAE$_r$ estimates are globally consistent. We report meaningful estimation patterns since the reported curves tend to get closer to the true value as the number of epochs increase, although it might show some bias when initial value is far from the true tail index value. The mean value of the estimated tail index is $1.56$ with a standard deviation of $0.2$.\\
We now focus on the angle generation. The best parameterization for the likelihood of the conditional VAE is a Dirichlet parameterization (see Appendix \ref{app: Dirichlet}). An important advantage of our approach is the ability to generate samples on the simplex for a given radius as detailed in Section \ref{sphere}, and even to sample the angular measure. Figure \ref{fig: projected density sphere vae} displays the angular measure projected onto the last two components of the simplex for
the true angular measure, our ExtVAE approach and the ParetoGAN. For the latter, we approximate the angular measure by the empirical measure above a very high threshold. The ExtVAE shows a good agreement with the true distribution, though not as sharp. 
By contrast, the distribution sampled by the ParetoGAN tends to reduce to a single mode. 
This confirms the result of Proposition \ref{thm: GAN and limit distribution}.\\
\begin{figure}
    \centering
    \includegraphics[width=\columnwidth]{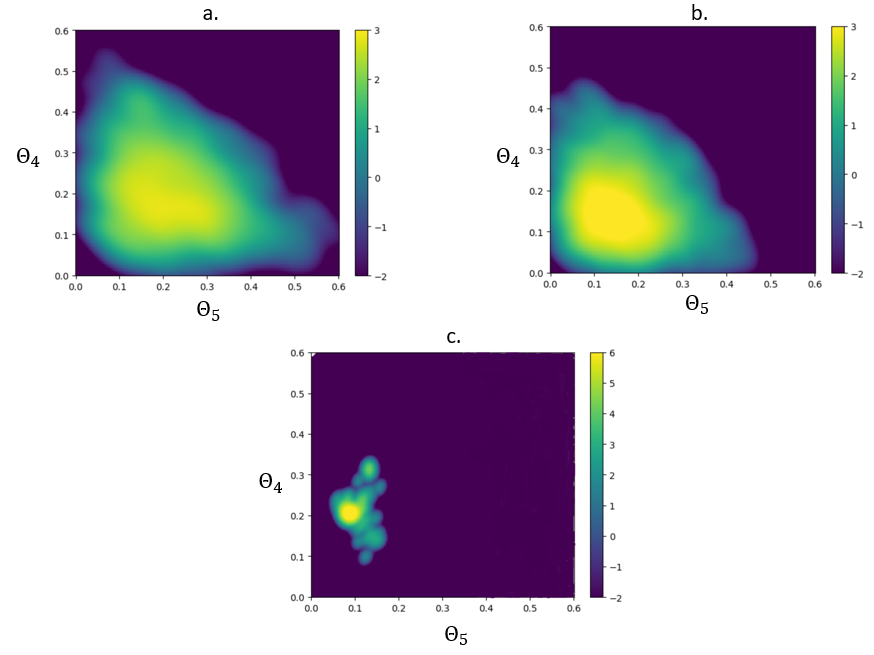}
    \caption{Log probability of the angular measure obtained with a. ExtVAE, b. true distribution, c. ParetoGAN, projected on axes 4 and 5 (named $\theta_4$ and $\theta_5$). For ParetoGAN, the estimation is based on 10000 samples at a high value of radius, typically above 10, which corresponds to the upper percentile of $R_1$ distribution.}
    \label{fig: projected density sphere vae}
\end{figure}
\noindent
To compare the overall distributions sampled by the the benchmarked schemes to the true data distribution, we estimate the Wasserstein distance (Equation \ref{eq: Wasserstein}) between 10000 generated items and the test set. The ExtVAE performs slightly better than the ParetoGAN (5.37 vs. 6.80). To compare the tails, Figure~\ref{fig: comp rescaled OT} shows the rescaled Wasserstein distance above a radius threshold (see Equation~\ref{eq: rescaled Wasserstein}), computed between samples generated by the two approaches and the true data. We focus on radius thresholds above 2, which corresponds to the highest decile of the true data. Again, the ExtVAE performs better than the ParetoGAN, especially for radius values between 2 and 4, corresponding roughly to quantiles between 0.90 and 0.95. We recall that 
the ParetoGAN relies on the minimization of a Wasserstein metric, whereas the 
ExtVAE relies on a likelihood criterion. Therefore, we regard these results as an illustration of the better generalization performance of the ExtVAE, especially for the extremes.\\
\begin{figure}
    \centering
    \includegraphics[width=0.6\columnwidth]{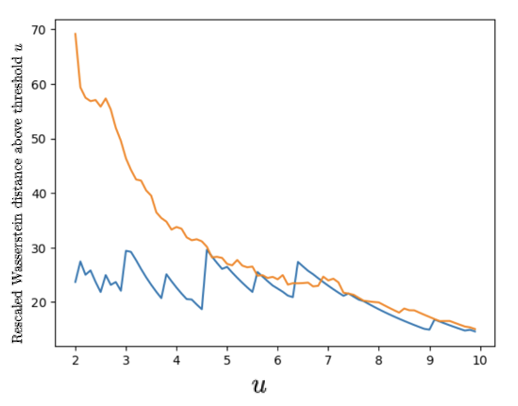}
    \caption{Wasserstein distance upon radius threshold $u$, divided by the square of $u$, calculated between 10000 samples drawn from generative models and test set. The ParetoGAN is shown in orange, and our model in blue. The considered thresholds are above $2$, which is roughly the upper decile of the radius distribution.}
    \label{fig: comp rescaled OT}
\end{figure}
\noindent
At last, we estimate the threshold at which the radius and angle distributions become independent according to the criterion proposed by \cite{wan2019threshold}. Although, by construction, there is no radius value from which there is a true independence, the estimator gives a radius above which one can approximately consider that the limit measure is reached.
We compare in Figure \ref{fig: pvalues} the $p$-values for assessing independence between the radius distribution and the angle distribution (see Appendix \ref{app: Threshold selection}). The $p$-values are represented as a function of the chosen threshold for each of the three considered data sets:  the test data set, the data set sampled from the ExtVAE and the data set sampled from the ParetoGAN. 
The ExtVAE slightly underestimates the radius threshold compared to the true data ($1.3$ vs. $1.6$), while the ParetoGAN leads to a large overestimatation ($2.7$ vs. $1.6$). 
This illustrates further that the ExtVAE better captures the statistical features of high quantiles than ParetoGAN does. This improvement may be attributed to the polar decomposition used in the ExtVAE, which enables a more accurate modelling of the joint asymptotic behaviour of the radius and the angle.
\begin{figure}
    \centering
    \includegraphics[width=0.6\columnwidth]{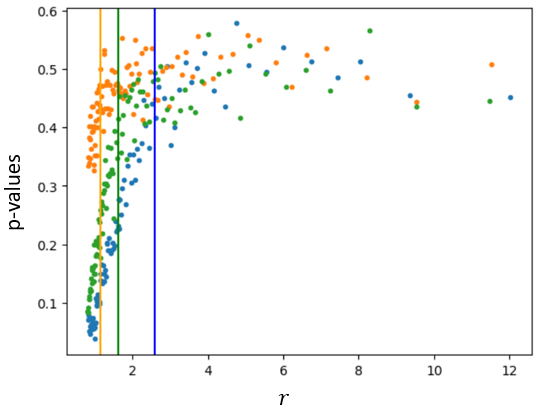}
    \caption{$p$-values for assessing independence between radius and angle distribution at different radius thresholds, computed according to Appendix \ref{app: Threshold selection} for the test data set (green points), 10000 samples of the ExtVAE (orange points), and 10000 samples of the ParetoGAN (blue points). The vertical bars correspond to the threshold below which the $p$-values are less than $0.45$. Above this threshold, the radius and the angle can be roughly considered as independent. We refer to \cite{wan2019threshold} for further details.}
    \label{fig: pvalues}
\end{figure}

\subsection{Danube River Discharge Case-study}
Our second experiment addresses a real heavy-tailed multivariate data set.
We consider the daily time series of river flow measurements over 50 years at five stations (among thirty-one) of the Danube river network (see Appendix \ref{app: Danube dataset} for further details). River flow data are known to often exhibit heavy-tailed behavior \citep[see][]{katz2002statistics}. In reference to the numbering of the stations (see Figure \ref{fig: Danube}), we note the random variables associated with the considered stations $X_{23},\ X_{24},\ X_{25},\ X_{26},\ X_{27}$. From the 50-year time series of daily measurements, we take a measurement every 25 days in the considered stations to form the training set. The remaining set serves as the test set. There are 730 daily measurements in the training set, and 17486 daily measurements in the test set. We have deliberately chosen a training set size that is significantly smaller than the test set size. This keeps the problem within the scope of tail extrapolation while retaining sufficient test data to reliably assess the accuracy of our tail distribution estimate. \\
We focus on the question raised in introduction (see Figure \ref{fig: pb extreme}): can we extrapolate and generate consistent samples in extreme areas not observed during the training phase? We focus on extreme areas of the form $\bigcap_{i=23}^{27} X_i >u_i$ with $u_i$ large predefined thresholds. This corresponds to flows exceeding predefined thresholds at several stations. Namely we define
\begin{equation}\label{eq: conjugate_extremes}
    A_j^{(p)}=\bigcap_{i=23}^{j} X_i >u_i^{(p)}, 
\end{equation}
with $p$ a given probability level and $u_i^{(p)}$ the corresponding quantile of the flow $i$ in test set. The estimation of the probabilities of occurrence of such events is key to the assessment of major flooding risks along the river. \\
Our experiments proceed as follows. We train the StdVAE, the UExtVAE and the ParetoGAN using the training set. For this case-study, the best parameterization for the likelihood of the angular part of the UExtVAE is a projection of a multivariate normal distribution (see Equation \ref{eq: ang param}). As evaluation framework, we generate for each trained model a number of samples of the size of the test data set, and we compare the proportion of samples that satisfy a given extreme event $A_j^{(p)}$ to that in the test data set.
We consider extreme events $A_j^{(p)}$ with $p=0.9$ and $p=0.99$.
Table \ref{tab: prop Danube} synthesizes this analysis for the StdVAE, UExtVAE and ParetoGAN.
Note that the training data set does not include the considered extreme events when  $p=0.99$. 
Interestingly, the UExtVAE samples such extreme events with the same order of magnitude of occurrence as in the test data set. For instance, the proportion of samples that satisfy $A_{26} ^{(0.99)}$ and $A_{27} ^{(0.99)}$ is consistent with that observed in the test data set, respectively 0.2\% and 0.18\% 
against $0.4$\% and $0.25$\%. 
In contrast, the StdVAE does not generalize beyond the training data set. While it generates extreme events corresponding to \(p = 0.9\), as these are present in the training data, it fails to produce any samples in \(A_{26}^{(0.99)}\) and \(A_{27}^{(0.99)}\).
ParetoGAN generates samples that satisfy $A_{25} ^{(0.99)}$, $A_{26} ^{(0.99)}$ and $A_{27} ^{(0.99)}$. Although satisfactory, the sampled proportions are further from true proportions than for our approach. Moreover, by repeating the experiment, it appears that $A_{25} ^{(p)}$ is always equal to $A_{26} ^{(p)}$ whenever $p\geq0.9$. The effect is likely due to extremes being generated along a specific direction, as established in Proposition \ref{thm: GAN and limit distribution}.\\
\begin{table}
\caption{\label{tab: prop Danube} 
Proportion (in \%) of extreme events $A_j^{(p)}$ (Equation \ref{eq: conjugate_extremes}) in the training and test data sets as well as in the data sets sampled from the trained StdVAE, UExtVAE and ParetoGAN with the same size as the test data set. The considered values of $j$ are $25$, $26$ and $27$, and those of $p$ are $0.9$ and $0.99$.} 
\begin{center}
\begin{tabular}{|c||c|c|c|c|c|}
    \hline
     &\multicolumn{5}{|c|}{$p = 0.9$} \\ 
    \hline
     & Train & Test & UExtVAE & StdVAE & ParetoGAN \\
    \hline
    \hline
   $A_{25} ^{(p)}$ & 5.9 & 6.6 & 5.0 & 3.8 & 5.5\\
   $A_{26} ^{(p)}$ & 4.9 & 6.0 & 4.6 & 3.3 & 5.5\\
   $A_{27} ^{(p)}$ & 3.8 & 5.1 & 4.1 & 2.5 & 4.4\\
    \hline
\end{tabular}

\bigskip

\begin{tabular}{|c||c|c|c|c|c|}
    \hline
     & \multicolumn{5}{|c|}{p = 0.99}\\ 
    \hline
     & Train & Test & UExtVAE & StdVAE & ParetoGAN \\
    \hline
    \hline
   $A_{25} ^{(p)}$ & 0.0 & 0.48 &  0.22 & 0.01 & 0.13\\
   $A_{26} ^{(p)}$ & 0.0 & 0.4 &  0.2 & 0.0 & 0.13\\
   $A_{27} ^{(p)}$ & 0.0 & 0.25 &  0.18 & 0.0 & 0.09\\
    \hline
\end{tabular}
\end{center}
\end{table}
Because we study a true data set, the tail index of the radius of the discharge data set is not known a priori. \cite{ref32} reports an estimate of $3.5 \pm 0.5$ considering only the summer months. In our case, the tail index of the trained UExtVAE is of $4.5$. It is slightly higher than the value found by \cite{ref32}, which means a less heavy-tailed distribution. Indeed, half of the annual maxima occurs in June, July or August, typically due to heavy summer rain events. Thus, summer months are expected to depict heavier tails than the all-season data set, which is consistent with our findings.

\section{Conclusion}\label{ccl}
This study bridges VAE and EVT to address the generative modeling of multivariate extremes. 
Following the concept of multivariate regular variation, we leverage a polar decomposition to combine a VAE for heavy-tailed radius sampling with a conditional VAE to sample from the angular distribution given the radius. 
Doing so, we explicitly address the dependence between the variables at each radius, and in particular the angular measure. 
Experiments performed on synthetic and real data support the relevance of our approach compared with vanilla VAE schemes and GANs tailored for extremes. In particular, we illustrate the ability to consistently sample extreme regions that have never been observed during the training stage.\\
Our contribution naturally advocates for extensions to multivariate extremes in time and space-time processes \citep{basrak2009regularly,liu2012sparse} as well as to VAE for conditional generation problems \citep{zheng2019pluralistic,grooms2021analog}.


\acks{The authors acknowledge the support of the French Agence Nationale de la Recherche (ANR) under reference ANR-Melody (ANR-19-CE46-0011). 
Part of this work was supported by 80 PRIME CNRS-INSU, ANR-20-CE40-0025-01 (T-REX project), and the European H2020 XAIDA (Grant agreement ID: 101003469). \\}


\newpage

\appendix
\section{Data Sets}\label{app: data set}
This appendix provides details on the two data sets used in the experiments. One data set is synthetic (\ref{app: synthesized dataset}) and the other is a true data set compiling flow measurements (\ref{app: Danube dataset}).
\subsection{Synthesized Data Sets}\label{app: synthesized dataset}
We sample in a space of dimension $5$. The radius distribution $R_1$ is given by
\begin{align*}
    R_1 = \ &2{\bf U} \times  {\bf Inv\Gamma}(\alpha=1.5\ ; \ \beta = 0.6), \\
\end{align*}
with ${\bf U}$ uniformly distributed on $[0,1]$. From Breiman's Lemma, the radius distribution is heavy-tailed with tail index $\alpha=1.5$.\\
The detailed expression of the conditional angular distribution ${\bf \Theta_1}\mid R_1 = r$ 
 is given by
\begin{equation}\label{eq: def R1 theta1}
    \big[{\bf \Theta_1}\mid R_1 = r\big] =  \mathbf{Diri}\left(a_1(r),a_1(r),a_2(r),a_2(r),a_2(r)\right),
\end{equation}
where $a_1(r) = 3\left(2 - \mathrm{min}(1,1/2r)\right)$, $a_2(r)=3\left(1 + \mathrm{min}(1,1/2r)\right)$, and $\mathbf{Diri}$ stands for Dirichlet distribution (see Appendix \ref{app: Dirichlet}).\\
Figure \ref{fig: Radius pdf} gives the empirical pdf of $R_1$ based on 1000 samples. 

\begin{figure}
\centering
\includegraphics[width=\columnwidth]{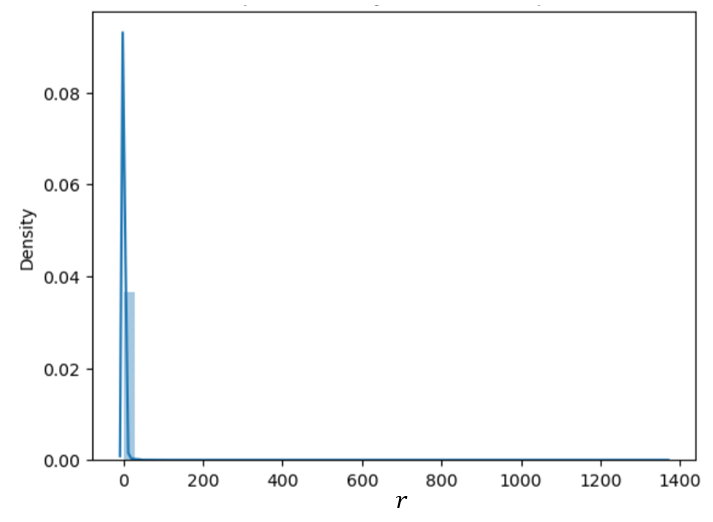}
\caption{Empirical density of $R_1$ based on 1000 samples.}
\label{fig: Radius pdf}
\end{figure}

\subsection{Danube River Network Discharge Measurements}\label{app: Danube dataset}
The Danube upper basin is a European river network which drainage basin covers a large part of Austria, Switzerland and southern Germany. Figure \ref{fig: Danube} shows the topography of the Danube basin as well as the locations of the 31 stations at which daily measurements of river discharge are available for a 50-year time period. Danube river network data set is available from the Bavarian Environmental Agency at \url{http://www.gkd.bayern.de}. As river discharges usually exhibit heavy-tailed distribution, this data set have been extensively studied in the community of multivariate extremes \citep[see, e.g.][]{ref32,ref31}. We consider measurements from a subset of five stations, numbered 23 to 27 (red triangles in Figure \ref{fig: Danube}), from which we aim to generate new realistic observations.
\begin{figure}
    \centering
    \includegraphics[width=\columnwidth]{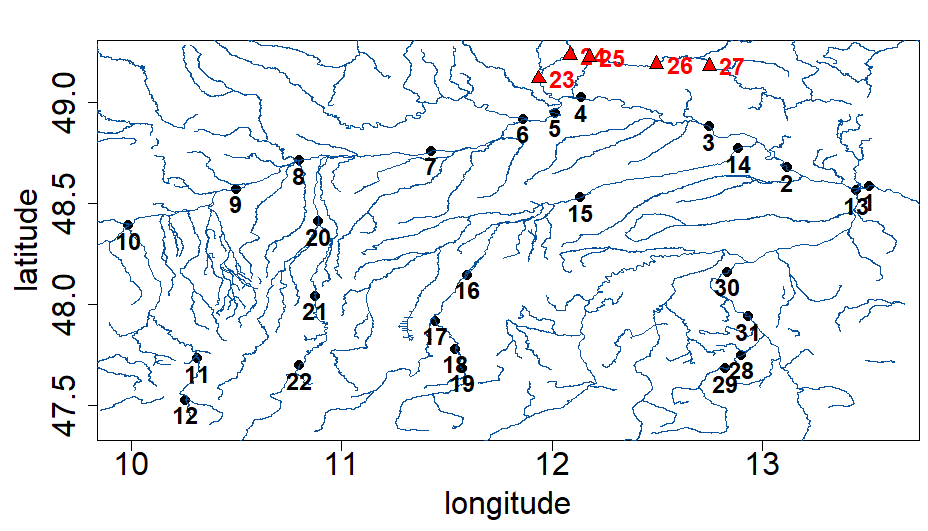}
    \caption{Topographic map of the upper Danube basin with 31 available gauging stations. In our experiments, we focus on the 5 stations indicated by the red triangles. A data set of 50 years of daily measurements is considered (from 1960 to 2010).}
    \label{fig: Danube}
\end{figure}
\section{Additional Notions}
This appendix provides further explanations of notions that are used either in the main text or in the proofs (Appendix \ref{app: proofs}).
\subsection{Lipschitz Continuity}\label{app: lip_continuity}
\begin{definition}\label{def: Lipschitz}
    Let $\left(\mathbb{E}, d_\mathbb{E}\right)$ and $\left(\mathbb{F}, d_\mathbb{F}\right)$ be two metric spaces with $d_\mathbb{E}$ and $d_\mathbb{F}$ the respective metrics on sets $\mathbb{E}$ and $\mathbb{F}$. A function $f: \mathbb{E} \rightarrow \mathbb{F}$ is called Lipschitz continuous if there exists a real constant $k \geqslant 0$ such that, for all $x_1$ and $x_2$ in $\mathbb{E}$, 
    \begin{equation}
        d_\mathbb{F}\left(f\left(x_1\right), f\left(x_2\right)\right) \leqslant k d_\mathbb{E}\left(x_1, x_2\right).
    \end{equation}
\end{definition}
\begin{remark}\label{rk: lip_fun} 
    If $\mathbb{E}$ and $\mathbb{F}$ are normed vector spaces with respective norm $\|.\|_{\mathbb{E}}$ and $\|.\|_{\mathbb{F}}$, then $f$ Lipschitz continuous implies that there exists $k\geq 0$ such that
    $$
    d_\mathbb{F}\left(f(x), f\left({\bf 0}_{\mathbb{E}}\right)\right) \leqslant k d_\mathbb{E}\left(x, {\bf 0}_{\mathbb{E}} \right).
    $$
    Consequently, $\|f(x)\|_{\mathbb{F}}\leqslant k \|x\|_{\mathbb{E}} + \|f({\bf 0}_{\mathbb{E}})\|_{\mathbb{F}}$. 
\end{remark}
\subsection{Convergence of Measures}\label{app: weak_conv}
\begin{definition}
    Let $\mathbb{E}$ be a metric space and $\left(\mu_n\right)_{n \in \mathbb{N}}$ be a sequence of measures, then $\mu_n$ converges weakly to a measure $\mu$ as $n\to \infty$, if, for any $f:\mathbb{E}\longrightarrow\mathbb{R}$ continuous real-valued bounded function,
    $$
    \lim_{n\to\infty}\int_{\mathbb{E}}fd\mu_n = \int_{\mathbb{E}}fd\mu.
    $$
    If this asymptotic result holds when $f$ is continuous with compact support, we say that $\mu$ converges vaguely to $\mu_n$, denoted $\overset{v}{\longrightarrow}$. The vague convergence is equivalent to the convergence on all (relatively compact) continuity sets.
\end{definition}
\subsection{Additional definition of multivariate regular variation}
The following definition of multivariate regularly varying vector extends Definition \ref{def: MRV}. We refer to \cite{resnick2007heavy} Chapter 6 for a detailed review.
\begin{definition}\label{def: equivalent MRV}
    A random vector $\bf X$ has multivariate regular variation if there exists a function $b\to \infty$ and a Radon measure $\mu_{\bf X}$ called the limit measure such that 
 \begin{equation}\label{eq: limit measure}
        \lim_{t \to \infty}t\mathbb{P}\left(\frac{\mathbf{X}}{b(t)}\in \bullet\right) \overset{v}{\longrightarrow} \mathbf{\mu_{X}}(\bullet).
\end{equation}
\end{definition}
Note that this definition is slightly more general than Definition \ref{def: equivalent MRV}, as ${\bf X}$ does not necessarily have positive components. 
\begin{remark}\label{rk: MRV}
If the random vector $\bf X$ has positive components, the angular measure $\bf S_X$ defined in Equation \eqref{eq: wan MRV} is directly related to the limit measure. Indeed, for any measurable subset $\mathbb{A}$ of the simplex, and any measurable subset $\mathbb{I}$ of $(0,\infty)$, it holds that 
$$
\mathbf{\mu_{X}}\circ T^{-1}(\mathbb{I},\mathbb{A}) = \nu_\alpha(\mathbb{I})\times \bf S_X(\mathbb{A}),
$$
where $T$ is the polar transform define for any vector $\bf x$ by $T({\bf x}) = \left(\|{\bf x}\|\ , \  \frac{{\bf x}}{\|{\bf x}\|} \right)$, and $\nu_\alpha$ a measure on  $(0,\infty)$ such that $\nu_\alpha( [t,\infty])= ct^{-\alpha}$, with $c$ and $\alpha$ strictly positive constants.\\
Thus, the angular measure is the limit measure projected onto the simplex.    
\end{remark}

\section{Dirichlet Parameterization of the Likelihood}\label{app: Dirichlet}
A Dirichlet distribution with $m\geq2$ strictly positive parameters $(a_i)_{i=1}^m$ is supported by the $(m-1)$-dimensional simplex. Its pdf  is defined by
    \begin{eqnarray*}
    f_{\bf Diri}({\bf x}\ ; \ (a_i)_{i=1}^m) = \frac{1}{B\left((a_i)_{i=1}^m\right)}\prod_{i=1}^m x_i^{a_i-1},\\
    \text{with } {\bf x}\in\mathbb{R}^m \text{ s.t } \sum_{i=1}^m x_i =1,
    \end{eqnarray*}
    where $B$ is the multivariate beta distribution.\\
To use a Dirichlet parameterization of the likelihood, Equation \eqref{eq: ang param} becomes
$$
p_\nu({\bf s}\mid {\bf z}_{ang},r)\sim f_{\bf Diri}({\bf s}\ ; \ a_{\nu}({\bf z}_{ang},r)),
$$
where $a_\nu$ takes values in $\left(\mathbb{R}^+\right)^m$.\\
Using this parameterization, Condition \ref{cond: indep ang rad} must be modified to preserve the asymptotic independence between the radius and angle distributions. It then takes the following form.
\begin{condition}
     $a_{\nu}$ is such that there exists a function $a_{\infty}$ which verifies for each ${\bf z}_{ang}$
    \begin{equation*}
        \lim_{r \to +\infty}a_{\nu}({\bf z}_{ang},r) = a_{\infty}(\mathbf{z}_{ang}).
    \end{equation*}
\end{condition}
Again, we set $a_{\nu}({\bf z}_{ang},r) =  h_{\nu}({\bf z}_{ang},\frac{1}{1+r})$, with $h_{\nu}$ Lipschitz continuous and $a_{\infty}(\mathbf{z}_{ang}) = h_\nu({\bf z}_{ang},0)$. Similar to Remark \ref{rk: param_inf}, it is still easy to sample from the angular measure.
\section{Tail Index Estimation}\label{app: tail index}
Estimating the tail index of a univariate distribution from samples is not an easy task. 
As an illustration, we drew the Hill plot (see e.g., \cite{resnick2007heavy}, Section~4.4; \cite{xie2017analysis}, Section~2.2) for $R_1$ in Figure~\ref{fig: Hill plot}. 
The Hill plot is based on the Hill estimator, which computes an estimate of the tail index from the $k$ largest order statistics. 
Plotting this estimator as a function of $k$ provides a graphical tool commonly used in the extreme-value community to assess the tail index. 
If the curve becomes approximately constant beyond a certain order statistic, this plateau value is taken as an estimate of the tail index. 
In our case, however, the Hill plot is difficult to interpret because no clear plateau emerges.

Other methods are also widely used to estimate the tail index, such as maximum likelihood estimation. 
This approach consists in fitting a GP distribution (Equation~\ref{eq: GP}) to the subset of data exceeding a chosen threshold \citep[see][for details]{coles2001introduction}. 
For example, on the training dataset of $R_1$, maximum likelihood estimation yields an estimate of $1.28$ for the tail index when the threshold is the $0.8$-quantile, but this estimate increases to $1.67$ when using the $0.9$-quantile as threshold.

\begin{figure}
    \centering
    \includegraphics[width=0.8\linewidth]{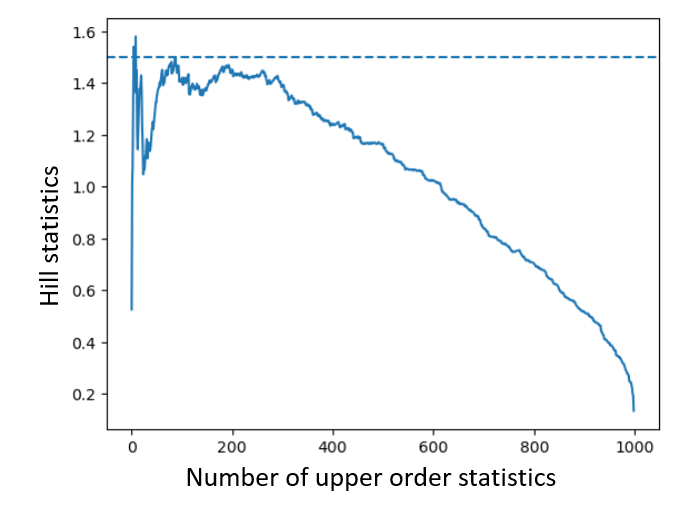}
    \caption{Hill plot for the 1000 $R_1$ samples of train and validation set (blue curve), the dashed line indicates the true value of the tail index, i.e. $1.5$. }
    \label{fig: Hill plot}
\end{figure}

\section{Criteria}\label{app: criteria}
This section details the evaluation criteria used to benchmark the approaches in Section~\ref{experiments}. These criteria assess (i) the radius distributions, with an emphasis on tail behavior (Appendix~\ref{app: KLu}), (ii) the similarity between the full multivariate distributions (Appendix~\ref{app: WD}), and (iii) key statistics of the angular distributions (Appendix~\ref{app: Threshold selection}).
\subsection{KL Divergence upon Threshold}\label{app: KLu}
Let us assume that we have $n$ samples ${\mathcal{R}_{true}} =(r_{true}^1,r_{true}^2, ..., r_{true}^n)$ from the true radius distribution and $m$ samples ${\mathcal{R}_{gen}} = (r_{gen}^1,r_{gen}^2, ..., r_{gen}^m)$ from a generative model. Let $\tilde{\bar{F}}_{true}$ (respectively $\tilde{\bar{F}}_{gen}$) denote the empirical estimator of the survival function of the true distribution (respectively the generated distribution) chosen to be non-zero above the largest observed value. Then the empirical estimate $\mathrm{KL}_u({\mathcal{R}_{true}},{\mathcal{R}_{gen}})$ of the KL divergence beyond a threshold $u$ is given by
\begin{eqnarray}\label{eq: KL div threshold}
    \mathrm{KL}_u({\mathcal{R}_{true}},{\mathcal{R}_{gen}}) = &-1 -\frac{1}{N_n}\sum_{i=1}^m\log\left( \frac{\tilde{\bar{F}}_{gen}(max(r_{gen}^i,u))}{\tilde{\bar{F}}_{gen}(u)}\right) \nonumber\\
    &- 1 - \frac{1}{M_m}\sum_{i=1}^n\log\left( \frac{\tilde{\bar{F}}_{true}(max(r_{true}^i,u))}{\tilde{\bar{F}}_{true}(u)}\right) ,
\end{eqnarray}
where $N_n$ and $M_m$ are the number of samples above threshold $u$ respectively among ${\bf \mathcal{R}_{true}}$ and ${\bf \mathcal{R}_{gen}}$. 

\subsection{Wasserstein Distance}\label{app: WD}
Assume we have a set $\mathcal{X} = (\mathbf{x}_1,\dots,\mathbf{x}_n)$ of $n$ i.i.d.\ samples from a random vector $\mathbf{X}\in\mathbb{R}^d$, and a set $\mathcal{Y} = (\mathbf{y}_1,\dots,\mathbf{y}_m)$ of $m$ i.i.d.\ samples from a random vector $\mathbf{Y}$ of the same dimension. 
We define $W(\mathcal{X},\mathcal{Y})$ as the empirical Wasserstein distance between the empirical measures supported on $\mathcal{X}$ and $\mathcal{Y}$:
\begin{equation}\label{eq: Wasserstein}
W(\mathcal{X},\mathcal{Y}) 
= \left( \min_{\gamma \in \mathbb{R}_+^{n\times m}}
      \sum_{i,j} \gamma_{i,j}\,\|\mathbf{x}_i - \mathbf{y}_j\|_2
   \right)^{1/2},
\end{equation}
subject to
\[
n\gamma \mathbf{1}_m = \mathbf{1}_n,
\qquad
m\gamma^\top \mathbf{1}_n = \mathbf{1}_m,
\]
where $\mathbf{1}_k$ denotes the $k$-dimensional vector of ones and $\|\cdot\|_2$ the Euclidean norm.
The rescaled Wasserstein distance above a threshold $r$ is defined by
\begin{equation}\label{eq: rescaled Wasserstein}
W_r(\mathcal{X},\mathcal{Y}) 
= \frac{W(\mathcal{X}_r,\mathcal{Y}_r)}{r^2},
\end{equation}
where $\mathcal{X}_r=\{\mathbf{x}_i\in\mathcal{X} : \|\mathbf{x}_i\|_2>r\}$ 
and similarly $\mathcal{Y}_r$.

\subsection{Threshold selection}\label{app: Threshold selection}
Let us consider ${\bf X}$ a random vector of $\mathbb{R}^d$ with a polar decomposition $(R,{\bf \Theta})$. Let $(\mathbf{x}_1,\mathbf{x}_2,\cdots,\mathbf{x}_N)$ be a sequence of observed samples of ${\bf X}$ with corresponding polar coordinates $r_i$ and $\theta_i$ for each $\mathbf{x}_i$.
Given a decreasing sequence of candidate thresholds for the radius $r_k^{(ref)}$, the aim is to find the smallest $r_k^{(ref)}$ such that $R$ and ${\bf \Theta}$ are independent given $R>r_k^{(ref)}$.
To assess independence between radius and angle distributions, \cite{wan2019threshold} relies on the following hypothesis testing framework:
\begin{itemize}
    \item $H_0$: $R/r_k^{(ref)}$ and ${\bf \Theta}$ are independent given $R > r_k^{(ref)}$,
    \item $H_1$: $R/r_k^{(ref)}$ and ${\bf \Theta}$ are not independent given $R >r_k^{(ref)}$.
\end{itemize}
Given these hypotheses, the authors propose a procedure for testing $H_0$ against $H_1$, that yields a $p$-value uniformly distributed under $H_0$ and small under $H_1$. Thus, for a given threshold $r_k$, when we average the $p$-values, we should find about $0.5$ if $H_0$ is true and closer to $0$ when $H_1$ is true.\\
To compute the $p$-values, the authors rely on the empirical distance covariance \citep{szekely2007measuring}.
\begin{definition}
The empirical covariance between $N$ observations $\mathcal{X} = \{ {\bf x_i} \}_{i=1}^N$ and $N$ observations $ \mathcal{Y} =  \{ {\bf y_i}\}_{i=1}^N$ of two random vectors $\bf{X}$ and ${\bf Y}$ is given by
\begin{align*}
    T_N(\mathcal{X},\mathcal{Y}) &= \frac{1}{N^2}\sum_{i,j=1}^{N}\|{\bf x_i}-{\bf x_j}\|_2\|{\bf y_i}-{\bf y_j}\|_2 + \frac{1}{N^4}\sum_{i,j,k,l=1}^{N}\|{\bf x_i}-{\bf x_j}\|_2\|{\bf y_k}-{\bf y_l}\|_2\\
    &- \frac{2}{N^3}\sum_{i,j,k=1}^{N}\|{\bf x_i}-{\bf x_j}\|_2\|{\bf y_i}-{\bf y_k}\|_2,
\end{align*}
with $\|.\|_2$ the euclidean distance. Notice that ${\bf X}$ and ${\bf Y}$ have not necessarily equal sizes.
\end{definition}
 For a fixed threshold $r_k^{(ref)}$, the following data sets are considered,
 \begin{align*}
    (\mathcal{R}_{dep} ,{\bf \Theta}_{dep})&=\{(r_i,{\theta}_i) \ \text{ s.t. }r_i>r_k^{(ref)}, \; 1\leq i \leq N\},\\
    \mathcal{R}_{indep}&=\{r_i \ \text{ s.t. }r_i>r_k^{(ref)}, \; 1\leq i \leq N\},\\
    {\bf \Theta}_{indep}&=\{{\theta}_i \ \text{ s.t. }r_i>r_k^{(ref)}, \; 1\leq i \leq N\}.
 \end{align*}
We randomly choose a subsample of $(\mathcal{R}_{dep} ,{\bf \Theta}_{dep})$ of size $N_k$ denoted $(\mathcal{R}_{dep}^{N_k} ,{\bf \Theta}_{dep}^{N_k})$. Then, we compute $T_{N,k} = T_{N_k}(\mathcal{R}_{dep}^{N_k} ,{\bf \Theta}_{dep}^{N_k})$, which is the empirical covariance between the radii and angles within the subsample.\\ 
To compute a $p$-value of $T_{N,k}$ under the assumption that the conditional empirical distribution is a product of the conditional marginals, we take a large number $L$ of subsamples of size $N_k$ of $\mathcal{R}_{indep}$ on the one hand, and of ${\bf \Theta}_{indep}$ on the other hand, respectively denoted $\mathcal{R}_{indep}^{N_k,l}$ and ${\bf \Theta}_{indep}^{N_k,l}$ for $1\leq l\leq L$. The empirical covariances $\{\tilde{T}_{N,k}^{l}\}_{l=1}^L= \{T_{N_k}(\mathcal{R}_{indep}^{N_k,l} ,{\bf \Theta}_{indep}^{N_k,l})\}_{l=1}^L$ between radii and angles can be computed.\\
The $p$-value $pv_k$ of $T_{N,k}$ is the empirical value of $T_{N,k}$ relative the $\{\tilde{T}_{N,k}^l\}_{l=1}^L$, i.e.
$$
pv_k = \frac{1}{L}\sum_{i=1}^L \mathbf{1}_{\mathbb{R}^+}(T_{N,k}-\tilde{T}_{N,k}^{l}),
$$
with $\mathbf{1}_{\mathbb{R}^+}$ the indicator function of the set of positive real numbers.\\
This process is repeated $m$ times, with different subsamples of $(\mathcal{R}_{dep} ,{\bf \Theta}_{dep})$ leading to $m$ estimates of $pv_k$. The considered $p$-value is then the mean of these estimates. If the radius and angular distribution are independent, the $p$-value should be around $0.5$, otherwise it is close to 0.

\section{Implicit Reparameterization}\label{app: implicit}
When optimizing the ELBO cost given in Equation \eqref{eq: ELBO_VAE}, explicit reparameterization (see Equation \ref{eq: reparametrization}) is not feasible for the framework proposed in Section \ref{radius}. Following \cite{figurnov2018implicit}, we therefore rely on an implicit reparameterization. It consists in differentiating the Monte Carlo estimator of $E_{q_\phi(Z\mid r^{(i)})}[{f(Z)}]$ using
\begin{equation*}
    \nabla_\phi E_{q_\phi(Z \mid r^{(i)})}[{f(Z)}] =-E_{q_\phi(Z \mid r^{(i)})}[\nabla_z f(z)\nabla_\phi F_{q_\phi}(z)(\nabla_z F_{q_\phi}(z))^{-1} ],
\end{equation*}
with $F_{q_\phi}$ the cdf of $q_{\phi}$. An implicit reparameterization of Gamma distribution, as well as inverse Gamma and many others, is available within a Tensorflow package named TensorflowProbability.\footnote{Details could be found at \url{https://www.tensorflow.org/probability}}

\section{Proofs}\label{app: proofs}
\subsection{Proof of Proposition \ref{thm: VAE margins}}
In the standard parameterization described in Example \ref{ex: standard VAE}, the prior and likelihood are given by
\begin{align*}
\bf{Z} &\sim \mathcal{N}(\textbf{0},{\bf I}_n),\\
X\mid [\bf{Z}=\bf{z}] &\sim \mathcal{N}\left(\mu_\theta({\bf{z}}),\sigma_\theta({\bf{z}})^2\right).\\
\end{align*}
The aim of the proof is to show that $\lim _{u \rightarrow+\infty} u^a P(X>u)=0$ for all $a>0$, in accordance with Remark \ref{rk: lt_ht}. \\

The survival function of $X$ is given by
\begin{align*}
 P(X>u)&=\int_{\bf z} P(X>u \mid {\bf Z}={\bf z}) p({\bf z}) d{\bf z} \\
& =\int_{\bf z} \left(\int_u^{+\infty} \frac{1}{\sqrt{2 \pi \sigma_\theta({\bf z})^2}} \exp \left(-\frac{\left(x-\mu_\theta({\bf z})\right)^2}{2\sigma_\theta({\bf z})^2}\right) d x\right) p({\bf z}) d\bf{z} \\
& = \int_{\bf z} \frac{1}{2}  \operatorname{erfc}\left(\frac{u-\mu_\theta({\bf z})}{\sqrt{2}\sigma_\theta({\bf z})}\right) p({\bf z})d\bf{z}, \\
&
\end{align*}
where $\operatorname{erfc}$ is the complementary error function defined for $y\in \mathbb{R}$ by $\operatorname{erfc}(y) = 1-\ \frac{2}{\sqrt{\pi}}\int_0^ye^{-t^2}dt.$\\

Letting $\Omega(u) = \{{\bf z} \in \mathbb{R}^n \text{ s.t. } u - \mu_\theta(\bf{z}) >0\}$, the previous expression could also be written
\begin{align*}
 P(X>u)&=\int_{ { \bf z}\in \Omega(u)} \frac{1}{2 } \operatorname{erfc}\left(\frac{u-\mu_\theta({\bf z})}{\sigma_\theta({\bf z})}\right) p({\bf z})d\bf{z}\\ 
 &+ \int_{ { \bf z}\in \overline{\Omega(u)}} \frac{1}{2} \operatorname{erfc}\left(\frac{u-\mu_\theta({\bf z})}{\sigma_\theta({\bf z})}\right) p({\bf z})d\bf{z},\\
& =f_1(u) + f_2(u).
\end{align*}
Notice first that $f_2(u)\leq P({\bf z} \in \overline{\Omega(u)})$. As $\mu_\theta$ is Lipschitz continuous, there exists constants $k>0$ and $b\in \mathbb{R}$ such that $\mu_\theta({\bf z}) \leq k\|{\bf z}\| + b$ (see Remark \ref{rk: lip_fun}).\\
It implies that, for $u>b$,
\begin{align}\label{eq: upbound_f2}
    f_2(u)&\leq P\left(\|{\bf z}\|\geq \frac{u-b}{k}\right) \nonumber\\
    &\leq \frac{1}{2} \operatorname{erfc}\left(\frac{u-b}{\sqrt{2}k}\right) \nonumber \\
    &\leq \exp\left(-\left(\frac{u-b}{\sqrt{2}k}\right)^2\right).
\end{align}
where we have used the inequality \citep{chiani2003new}
\begin{equation}\label{eq: erfc_ineq}
    \operatorname{erfc(y)}\leq e^{-y^2}, \text{ for } y>0.
\end{equation}
From Equation \eqref{eq: upbound_f2}, we obtain $\lim _{u \rightarrow+\infty} u^a f_2(u)=0$.\\

To obtain a similar result for $f_1$, we again use inequality \eqref{eq: erfc_ineq} to obtain

$$
f_1(u) \leqslant \int_{{\bf z} \in \Omega(u)} \exp \left(-\left(\frac{u-\mu_\theta({\bf z})}{\sqrt{2}\sigma_\theta({\bf z})}\right)^2\right) p({\bf z})d{\bf z}.
$$
As $\sigma_\theta$ is Lipschitz continuous, there exist constants $k' >0$ and $b' \in \mathbb{R}$ such that $\sqrt{2}\sigma_\theta({\bf z}) \leq k'\|{\bf z}\| + b'$. Then, we can state that
\begin{equation}\label{eq: lastone}
f_1(u) \leqslant \int_{{\bf z} \in \Omega(u)} \exp \left(-\left(\frac{u-(k\|{\bf z}\| + b)}{k^{\prime} \|{\bf z}\| +b^{\prime}}\right)^2\right) p({\bf z}).
\end{equation}
For any given $a>0$, we define  the real-valued function
$$g_u(t)=u^a \exp \left(-\left(\frac{u-(kt + b)}{k^{\prime}t+b^{\prime}}\right)^2\right).$$ 
 The following holds :
$$
\lim _{u \rightarrow+\infty} g_u(t)=0.
$$
Additionally, $g_u(t)$ is maximized with respect to $u$ when $u=u^*(t)$ with
$$
u^*(t)=\frac{-(kt+b) + \sqrt{(kt+b)^2+4 a\left(k^{\prime}t +b^{\prime}\right)^2}}{2}.
$$
Thus, we obtain
$$
\left|g_u(t) \right|  \leqslant\left(2\sqrt{a}(k^{\prime }t+b^{\prime})\right)^{a},
$$
which leads to 
$$
\left|g_u(\|{\bf z}\|) p({\bf z})\right|  \leqslant \left(2\sqrt{a}(k^{\prime }\|{\bf z}\|+b^{\prime})\right)^{a} \exp{(-\|{\bf z}\|^2/2)}.
$$
Since the quantity on the right-hand is integrable with respect to $\|{\bf z}\|$ and independent of $u$, it follows by the dominated convergence theorem and Equation \eqref{eq: lastone} that $\lim _{u \rightarrow+\infty} u^a f_1(u)=0$.\\

The previous results yield $\lim _{u \rightarrow+\infty}  u^a P(X>u) =\lim _{u \rightarrow+\infty}u^a (f_1(u)+f_2(u))=0$. From  Remark \ref{rk: lt_ht}, we therefore conclude that $X$ is light-tailed.
\subsection{Proof of Proposition \ref{thm: GAN and limit distribution}}

In this proof, we make extensive use of the limit measure of a multivariate regularly varying vector introduced in Definition~\ref{def: equivalent MRV}. Recall from Remark~\ref{rk: MRV} that, when the components of a random vector are non-negative, the angular measure defined in Equation~\eqref{eq: wan MRV} is the projection of its limit measure onto the simplex. Consequently, to prove that the angular measure is supported by finitely many points of the simplex, it is sufficient to show that the limit measure is concentrated on finitely many direction vectors.

\medskip
\noindent\textbf{Step 1: Limit measure of the prior.}
First, the limit measure $\mu_{\mathbf{Z}}$ of ${\bf Z}$ is concentrated on the coordinate axes. More precisely,
\[
\mu_{\mathbf{Z}}\!\left( \mathbb{R}^n \setminus \bigcup_{i=1}^n \{ t \mathbf{e}_i : t>0 \} \right) = 0,
\]
where $\mathbf{e}_i$ denotes the $i$-th canonical basis vector of $\mathbb{R}^n$ (with a $1$ in the $i$-th position and $0$ elsewhere).  
As shown in \cite[Section~6.5]{resnick2007heavy}, this follows from the i.i.d.\ nature of the components of ${\bf Z}$.

\medskip
\noindent\textbf{Step 2: Effect of linear transformations.}
The following lemma describes how the limit measure behaves under multiplication by a matrix.

\begin{lemma}\label{lem: matmul}
Let ${\bf Y}$ be a $d$-dimensional regularly varying random vector whose limit measure is supported on the finite set of direction vectors 
$\bigcup_{i=1}^{n'} \{ t \mathbf{v}_i : t>0\}$.  
For any $m\times d$ matrix $\mathbf{W}$, the transformed vector ${\bf WY}$ is also regularly varying, and its limit measure is supported on 
\[
\bigcup_{i=1}^{n'} \{ t \mathbf{W}\mathbf{v}_i : t>0 \},
\]
excluding those indices $i$ for which $\mathbf{W}\mathbf{v}_i = \mathbf{0}$.  
In particular, the limit measure of ${\bf WY}$ is supported on at most $n'$ direction vectors.
\end{lemma}

\begin{proof}
It follows from
\[
t\,\mathbb{P}\!\left(\frac{\mathbf{WY}}{b(t)}\in \bullet \right)
= t\,\mathbb{P}\!\left(\frac{\mathbf{Y}}{b(t)}\in \mathbf{W}^{-1}(\bullet)\right)
\overset{v}{\longrightarrow} \mu_{\mathbf{Y}}( \mathbf{W}^{-1}(\bullet))
=: \mu_{\mathbf{WY}}(\bullet),
\]
that $\mathbf{WY}$ is regularly varying.

If $\mu_{\mathbf{Y}}$ is supported on $\{t \mathbf{v}_i\}$, then $\mu_{\mathbf{Y}}\circ \mathbf{W}^{-1}(\mathbb{A})>0$ implies 
\(
\mathbf{W}^{-1}(\mathbb{A}) \cap \{t \mathbf{v}_i\} \neq \emptyset
\),
hence 
\(
\mathbb{A} \cap \{t\,\mathbf{W}v_i\} \neq \emptyset
\),
provided $\mathbf{W}\mathbf{v}_i \neq 0$.  
Thus the support of $\mu_{\mathbf{WY}}$ is precisely the set of nonzero images $\mathbf{W}\mathbf{v}_i$.
\end{proof}

\medskip
\noindent\textbf{Step 3: Effect of adding a bias.}
\begin{lemma}\label{lem: bias}
If the $d$-dimensional vector ${\bf Y}$ has multivariate regular variation with limit measure $\mu_{\bf Y}$, then for any $d$-dimensional vector ${\bf b}$, the shifted vector ${\bf Y}+{\bf b}$ is also regularly varying, with the same limit measure.  
Hence, if $\mu_{\bf Y}$ is supported on finitely many direction vectors, the same holds for the limit measure of ${\bf Y}+{\bf b}$.
\end{lemma}

\begin{proof}
Write $\varepsilon_t = {\bf b}/b(t)\to 0$.  
Let $\mathbb{A}$ be a relatively compact continuity set, bounded away from the origin. Then
\[
t\,\mathbb{P}\!\left(\frac{\mathbf{Y}+\mathbf{b}}{b(t)}\in \mathbb{A}\right)
= t\,\mathbb{P}\!\left(\frac{\mathbf{Y}}{b(t)}\in \mathbb{A}-\varepsilon_t\right).
\]
Define for $\delta > 0$ the sets
\[
\mathbb{A}^{-\delta} := \{x : x + u \in \mathbb{A} \text{ for all } u, \|u\| \le \delta \}, 
\qquad
\mathbb{A}^{+\delta} := \{x : \operatorname{dist}(x, \mathbb{A}) \le \delta\}.
\]
 When $\|\varepsilon_t\|\le \delta$,  
\[
\mathbb{A}^{-\delta} \subset \mathbb{A}-\varepsilon_t \subset \mathbb{A}^{+\delta}.
\]
Using the vague convergence of the rescaled measures
\(
t\,\mathbb{P}\!\left(\frac{\mathbf{Y}}{b(t)}\in\bullet\right)
\) to $\mu_{\mathbf{Y}},$
we obtain, letting $t\to\infty$,

\[
\mu_{\bf Y}(\mathbb{A}^{-\delta})
\le \liminf_{t\to\infty}t\,\mathbb{P}\!\left(\frac{\mathbf{Y}+\mathbf{b}}{b(t)}\in \mathbb{A}\right)
\le \limsup_{t\to\infty}t\,\mathbb{P}\!\left(\frac{\mathbf{Y}+\mathbf{b}}{b(t)}\in \mathbb{A}\right)
\le \mu_{\bf Y}(\mathbb{A}^{+\delta}).
\]
Since $\mu_{\bf Y}(\partial \mathbb{A})=0$ and $\mathbb{A}$ is bounded away from $0$, the regularity of Radon measures ensures that  
$\mu_{\bf Y}(\mathbb{A}^{\pm\delta})\to \mu_{\bf Y}(\mathbb{A})$.  
Hence the limit measure is preserved.
\end{proof}

\medskip
\noindent\textbf{Step 4: Effect of the ReLU activation.}
Let $\mathrm{ReLU}(x)=\max(x,0)$ applied componentwise. If the limit measure of ${\bf Y}$ is supported on $\{t v_i\}$, then the limit measure of $\mathrm{ReLU}({\bf Y})$ is supported on the directions  
\[
\{t\,\mathrm{ReLU}(v_i)\ :\ t>0\}.
\]
If $v_i$ has at least one negative coordinate, the corresponding coordinate in $\mathrm{ReLU}(v_i)$ becomes $0$, potentially collapsing the direction.  
Thus ReLU can only reduce the number of supporting direction vectors.

\medskip
\noindent\textbf{Step 5: Conclusion for neural networks.}
A ReLU neural network consists of successive applications of:
\begin{itemize}
\item multiplication by weight matrices,
\item addition of biases,
\item ReLU nonlinearities.
\end{itemize}
By iterating Lemmas~\ref{lem: matmul}--\ref{lem: bias} and the ReLU argument, the limit measure remains concentrated on finitely many direction vectors at every layer.  
Since the output has non-negative coordinates, its angular measure is supported by a finite set of points of the simplex.  
Moreover, the number of such points is bounded above by the input dimension $n$.

This completes the proof.

\subsection{Proof of Proposition \ref{prop: breiman}}
    Consider an exponential distribution $A$, with scale parameter $c$, and $Z_{rad}$ an inverse-gamma distribution with parameters $(\alpha,\beta)$. The cdf of $R$ is given by
    \begin{align*}
        P(R\leq t) &= \int_0^{+\infty}P(A\leq z) \times \frac{t}{z^2}f_{\bf Inv\Gamma}\left(\frac{t}{z}\ ;\ \alpha,\beta\right)dz, \\
                    &= 1 - \frac{\beta^\alpha}{t^\alpha \Gamma(\alpha)} \int_0^{+\infty}z^{\alpha-1}e^{-\frac{z}{c}}e^{-\frac{\beta z}{t}},\\
                    &= 1- \frac{\beta^\alpha}{t^\alpha}\left(\frac{1}{c}+\frac{\beta}{t}\right)^{-\alpha},\\
                    &= 1 - \left(1+ \frac{t}{\beta c}\right)^{-\alpha},\\
                    &= 1- \bar{H}_{\sigma,\xi}(t),
    \end{align*}
    with $\sigma = \frac{\beta c }{\alpha}$ and $\xi = \frac{1}{\alpha}$. Consequently, $R$ follows a generalized Pareto distribution. Notice that we use the change of variable $u=z\left(\frac{1}{c}+\frac{\beta}{t}\right)$ from the second to the third line of the above equations. 

\subsection{Proof of Proposition \ref{prop: Radii VAE}}
The aim of this proof is to show that the radius $R$ sampled by the VAE is heavy-tailed with tail index $\alpha$. We first establish that there exists $\gamma > 0$ such that $E[R^\gamma] = \infty$, implying that $R$ is necessarily heavy-tailed, in accordance with Remark \ref{rk: lt_ht}. We then argue that the only possible value of the tail index is $\alpha$.\\

The pdf of $R$ is given by
\begin{align*}
    p(r) &= \int_0^{+\infty} p_{\theta}(r \mid z_{rad})p_\alpha(z_{rad})dz_{rad},\\
        &= \int_0^{+\infty} f(r,z_{rad})dz_{rad},\\
\end{align*}
with 
\begin{align*}
    f(r,z_{rad}) &= f_{\mathbf{\Gamma}}\left(r\ ;\ \alpha_\theta(z_{rad}),\beta_\theta(z_{rad})\right)f_{\bf Inv\Gamma}(z_{rad}\ ;\ \alpha ,1).\\
\end{align*}
Since $\alpha_\theta$ is equal to a constant denoted $a_\theta$, it leads to the expression
\begin{equation*}
    f(r,z_{rad})= \frac{r^{a_\theta -1}}{\Gamma(\alpha)\Gamma(a_\theta)} z_{rad}^{-(\alpha+1)}\beta_{\theta}(z_{rad})^{a_\theta}e^{-r\beta_\theta(z_{rad}) -\frac{1}{z_{rad}}}.
\end{equation*}

From Equations \eqref{eq: limit beta} and \eqref{eq: limit beta 2}, we can state the existence of $m$ and $M$ two strictly positive constants such that, for any $z$,
\begin{equation*}
    \frac{m}{z}\leq \beta_\theta(z) \leq \frac{M}{z}.
\end{equation*}
Consequently,  
\begin{equation*}
    f_1(r,z_{rad})\leq f(r,z_{rad})\leq f_2(r,z_{rad})
\end{equation*}
with
\begin{align*}
    f_1(r,z_{rad}) &= \frac{r^{a_\theta -1}}{\Gamma(\alpha)\Gamma(a_\theta)} z_{rad}^{-(a_\theta + \alpha+1)}m^{a_\theta}e^{-r\frac{M}{z_{rad}} -\frac{1}{z}},\\
    f_2(r,z_{rad})&= \frac{r^{a_\theta -1}}{\Gamma(\alpha)\Gamma(a_\theta)} z_{rad}^{-(a_\theta + \alpha+1)}M^{a_\theta}e^{-r\frac{m}{z_{rad}} -\frac{1}{z}}.
\end{align*}
We can obtain analytical expressions of $\int_0^{+\infty}f_1(r,z_{rad})dz_{rad}$ and $\int_0^{+\infty}f_2(r,z_{rad})dz_{rad}$, 
\begin{align*}
    \int_0^{+\infty}f_1(r,z_{rad})dz_{rad} &= \frac{r^{a_\theta -1}m^{a_\theta}}{\Gamma(\alpha)\Gamma(a_\theta)} \int_0^{+\infty} z_{rad}^{-(a_\theta + \alpha+1)}e^{-r\frac{M}{z_{rad}} -\frac{1}{z_{rad}}}dz_{rad},\\
    &= \frac{r^{a_\theta -1}m^{a_\theta}}{\Gamma(\alpha)\Gamma(a_\theta)}\Gamma(a_\theta+\alpha)(1+rM)^{-a_\theta-\alpha},
\end{align*}
where we used the change of variables $u = \frac{1+rM}{z_{rad}}$. Using same arguments, we also obtain
\begin{equation*}
    \int_0^{+\infty}f_2(r,z_{rad})dz_{rad} = \frac{r^{a_\theta -1}M^{a_\theta}}{\Gamma(\alpha)\Gamma(a_\theta)}\Gamma(a_\theta+\alpha)(1+rm)^{-a_\theta-\alpha}.
\end{equation*}
It leads to these asymptotic results when $r \to \infty$,
\begin{align*}
    \int_0^{+\infty}f_1(r,z_{rad})dz_{rad} \propto r^{-\alpha-1},\\
    \int_0^{+\infty}f_2(r,z_{rad})dz_{rad} \propto r^{-\alpha-1}.
\end{align*}
Consequently, $E[R^\gamma]$ is infinite whenever $\gamma>\alpha$. $R$ is necessarily heavy-tailed. Since $E[R^\gamma]$ is finite for $\gamma<\alpha$, the only possible value of the tail index of $R$ is $\alpha$ \citep[see][Proposition A.38 (d)]{embrechts1999extreme}.
\subsection{Proof of Proposition \ref{prop: DKL InvGamma}}
Let $\alpha_1$, $\alpha_2$, $\beta_1$ and $\beta_2$ be strictly positive constants. The following holds.
\begin{align}
    D_{KL}\left(f_{\bf Inv\Gamma}(z;\alpha_1 ,\beta_1 )\mid \mid f_{\bf Inv\Gamma}(z;\alpha_2 ,\beta_2 )\right) &= E_{z\sim{\bf Inv\Gamma}(\alpha_1 ,\beta_1 )}\left[\log\left(\frac{f_{\bf Inv\Gamma}(z;\alpha_1 ,\beta_1 )}{f_{\bf Inv\Gamma}(z;\alpha_2 ,\beta_2 )} \right)\right] \nonumber\\
    &= E_{y\sim{\bf \Gamma}(\alpha_1 ,\beta_1 )}\left[\log\left(\frac{f_{\bf Inv\Gamma}(\frac{1}{y};\alpha_1 ,\beta_1 )}{f_{\bf Inv\Gamma}(\frac{1}{y};\alpha_2 ,\beta_2 )} \right)\right] \nonumber \\
    &= E_{y\sim{\bf \Gamma}(\alpha_1 ,\beta_1 )}\left[\log\left(\frac{f_{\bf \Gamma}(y;\alpha_1 ,\beta_1 )}{f_{\bf \Gamma}(y;\alpha_2 ,\beta_2 )} \right)\right]\nonumber \\
    &= D_{KL}\left(f_{\bf \Gamma}(z;\alpha_1 ,\beta_1 )\mid \mid f_{\bf \Gamma}(z;\alpha_2 ,\beta_2 )\right). \label{eq: proof prop DKL}
\end{align}
Equation \eqref{eq: DKL analytic} holds from Equation \eqref{eq: proof prop DKL} and the following result \citep{penny2006variational}:
\begin{eqnarray*}
D_{KL}\left(f_{\bf \Gamma}(z;\alpha_1 ,\beta_1 )\mid \mid f_{\bf \Gamma}(z;\alpha_2 ,\beta_2 )\right) = (\alpha_1 - \alpha_2)\psi(\alpha_1) - \log\frac{\Gamma(\alpha_1)}{\Gamma(\alpha_2)} \nonumber \\
    +\alpha\log\frac{\beta_1}{\beta_2} + \alpha_1\frac{\beta_2-\beta_1}{\beta_1}.
\end{eqnarray*}

\vskip 0.2in
\bibliography{biblio}

\end{document}